%% file: nips_comparison.tex
\documentclass{article}

\usepackage[utf8]{inputenc} 
\usepackage[T1]{fontenc}    
\usepackage{hyperref}       
\usepackage{url}            
\usepackage{bbm}
\usepackage{booktabs}       
\usepackage{amsfonts}       
\usepackage{nicefrac}       
\usepackage{microtype}      

\usepackage{times}
\usepackage{graphicx} 
\usepackage{subfigure}
\usepackage{caption}

\usepackage{fullpage}
\usepackage{amsmath}
\usepackage{amsthm}
\usepackage{amssymb}
\usepackage{tikz}
\usepackage{xcolor}
\usetikzlibrary{arrows}

\usepackage{algorithm}
\usepackage[noend]{algpseudocode}
\usepackage{algorithmicx}
\usepackage{hyperref}
\usepackage{bm}

\newcommand{\sign}{\text{sign}}
\newcommand{\comment}[1]{} 
\newcommand{\dd}{\displaystyle}
\usepackage{accents}
\makeatletter
\def\widebar{\accentset{{\cc@style\underline{\mskip10mu}}}}
\def\Widebar{\accentset{{\cc@style\underline{\mskip8mu}}}}
\makeatother
\newcommand{\DIS}{\textsf{DIS}}
\newcommand{\SC}{\textup{\textsf{SC}}}
\newcommand{\Tol}{\textup{\textsf{Tol}}}
\makeatletter
\newcommand{\newalgname}[1]{%
	\renewcommand{\ALG@name}{#1}%
}
\newtheorem{theorem}{Theorem}

\newtheorem{lemma}[theorem]{Lemma}

\renewcommand{\O}{\mathcal{O}}
\renewcommand{\P}{\mathcal{P}}
\newcommand{\tO}{\tilde{\mathcal{O}}}
\newcommand{\E}{\mathbb{E}}
\newcommand{\err}{\textup{\textsf{err}}}
\newtheorem{condition}[theorem]{Condition}

\newcommand{\comp}{\textup{\text{comp}}}
\newcommand{\lab}{\textup{\text{label}}}
\newcommand{\pad}{\vspace{-4mm}}
\newcommand{\spad}{\vspace{-2mm}}
\newcommand{\sspad}{\vspace{-1mm}}
\newcommand{\bC}{\mathbb{C}}
\usepackage[shortlabels]{enumitem}
\usepackage{mathtools}
\allowdisplaybreaks[1]

\algrenewcommand\algorithmicrequire{\textbf{Input:}}
\algrenewcommand\algorithmicensure{\textbf{Output:}}

\title{Noise-Tolerant Interactive Learning Using\\ Pairwise Comparisons}

\author{
	Yichong Xu\thanks{Carnegie Mellon University. Email: yichongx@cs.cmu.edu} \and
	Hongyang Zhang\thanks{Carnegie Mellon University. Email: hongyanz@cs.cmu.edu} \and
	Kyle Miller\thanks{Carnegie Mellon University. Email: mille856@andrew.cmu.edu} \and	
	Aarti Singh\thanks{Carnegie Mellon University. Email: aarti@cs.cmu.edu} \and
	Artur Dubrawski\thanks{Carnegie Mellon University. Email: awd@cs.cmu.edu} 
}

\begin{document}

\maketitle
\pad
\begin{abstract}
We study the problem of interactively learning a binary classifier using noisy labeling and pairwise comparison oracles, where the comparison oracle answers which one in the given two instances is more likely to be positive. Learning from such oracles has multiple applications where obtaining direct labels is harder but pairwise comparisons are easier, and the algorithm can leverage both types of oracles. In this paper, we attempt to characterize how the access to an easier comparison oracle helps in improving the label and total query complexity. We show that the comparison oracle reduces the learning problem to that of learning a threshold function. We then present an algorithm that interactively queries the label and comparison oracles and we characterize its query complexity under Tsybakov and adversarial noise conditions for the comparison and labeling oracles. Our lower bounds show that our label and total query complexity is almost optimal.
\end{abstract}
\input{Introduction}

\input{Setup}

\input{adgac}
\input{General}

\input{Halfspace}
\input{LowerBounds}

\input{Conclusion}
\section*{Acknowledgements}
We thank Chicheng Zhang for insightful ideas on improving results in \cite{awasthi2017power} using Rademacher complexity.

{\small
\bibliographystyle{abbrv}
\bibliography{yichongref}
}
\clearpage

\appendix

\section{Our Techniques}

\medskip
\noindent{\textbf{Intransitivity:}} The main challenge of learning with pairwise comparisons is that the comparisons might be \emph{asymmetric} or \emph{intransitive}. If we construct a classifier $h(x)$ by simply comparing $x$ with a fixed instance $\hat{x}$ by comparison oracle, then the concept class of classifiers $\{h: h(x)=Z(x,\hat{x}), \hat{x}\in \mathcal{X}\}$ will have infinite VC dimension, so the complexity will be as high as infinite if we apply the traditional tools of VC theory. To resolve the issue, we conduct a group-based binary search in ADGAC. The intuition is that by dividing the dataset into several ranked groups $S_1,S_2,...$, the majority of labels in each group can be stably decided if we sample enough examples from that group. Therefore, we are able to reduce the original problem in the high-dimensional space to the problem of learning a ``threshold'' function in one-dimension space. Then some straightforward approaches such as binary search learns the thresholding function.


\medskip
\noindent{\textbf{Combining with Active Learning Algorithms:}} If the labels follow Tsybakov noise (i.e., Condition \ref{cond:labeltsy}), the most straightforward method to combine ADGAC with existing algorithms is to combine ADGAC with an algorithm that uses the label oracle only and works under TNC. However, we cannot save query complexity if we follow this method. To see this, notice that in each round we need roughly $n_i=\tilde{\O}\left(d\theta\left(\frac{1}{\varepsilon_i}\right)^{2\kappa-1}\right)$ samples and $m_i=\tilde{\O}\left(d\theta\left(\frac{1}{\varepsilon_i}\right)^{2\kappa-2}\right)$ labels; if we use ADGAC, we can obtain a labeling of $n_i$ samples with at most $\varepsilon_i n_i\approx m_i$ errors with low label complexity. Suppose $N$ is the set of labels that ADGAC makes error on. However, since the outside active learning algorithm works under TNC, we will need to query labels in $N$ to make sure that the ADGAC labels follow TNC. That means our label complexity is still $m_i$, the same as the original algorithm. To avoid this problem, we combine ADGAC with algorithms under adversarial noise in all cases including TNC. This eliminates the need to query additional labels, and also reduces the query complexity.

\medskip
\noindent{\textbf{Handling Independence:}} We mostly follow previous works on combining ADGAC with existing algorithms. However, since we now obtain labels from ADGAC instead of $\P_{\mathcal{XY}}$, the labels are not independently sampled, and we need to adapt the proof to our case. We use different methods for A$^2$-ADGAC and Margin-ADGAC: For the former, we use results from PAC learning to bound the error on all $n_i$ samples; for the latter, we decompose the error of any classifier $h$ on labels generated by ADGAC into two parts: The first part is caused by the error of ADGAC itself, and second is by $h$ on truthful labels. Using the above techniques enables us to circumvent the independence problem. 

\medskip
\noindent\textbf{Lower Bounds:} It is typically hard to provide a unified lower bound for multi-query learning framework, as several quantities are simultaneously involved in the analysis, e.g., the comparison complexity, the label complexity, the noise tolerance, etc. So traditional proof techniques for active learning, e.g., Le Cam's and Fano's bounds~\cite{castro2008minimax,hanneke2014theory}, cannot be trivially applied to our setting. 
Instead, we prove lower bounds on one quantity by allowing arbitrary budgets of other quantities. Another non-trivial technique is in the proof of minimax bound for the adversarial noise level of comparison oracle (see Theorem \ref{thm:minimaxnup}): In the proof of upper bound, we divide the integral region w.r.t. the expectation into $n$ segments, each of size $1/n$, and the expectation is thus the limit when $n\rightarrow \infty$. We upper bound the discrete approximation of the integral by a careful calibration of noise on each segment for a fixed $n$, and then let $n\rightarrow \infty$. The proof then leads to a general inequality (Lemma \ref{lemma:inequality}), and it might be of independent interest.

\section{Additional Related Work}
\label{sec:related}
It is well known that people are better at comparison than labeling \cite{stewart2005absolute,shah2014better}. It has been widely used to tackle problems in classification \cite{maji2014part}, clustering \cite{krishnamurthy2015interactive} and ranking \cite{agarwal2009generalization,furnkranz2010preference}. 
Balcan et al.~\cite{balcan2016learning} studied using pairwise comparisons to learn submodular functions on sets. Another related problem is bipartite ranking \cite{agarwal2005stability}, which exactly does the opposite of our problem: Given a group of binary labels, learn a ranking function that rank positive samples higher than negative ones.

Interactive learning has wide application in the field of computer vision and natural language processing (see e.g., \cite{wah2014similarity}). There are also abundant literatures on interactive ways to improve unsupervised and semi-supervised learning \cite{krishnamurthy2015interactive}. However, there lacks a general statistical analysis of interactive learning for traditional classification tasks. Balcan and Hanneke~\cite{balcan2012robust} analyze class conditional queries (CCQ), where the user gives counterexamples to a given classification. Beygelzimer et al.~\cite{beygelzimer2016search} used a similar idea using search queries. However, their interactions requires a oracle that is usually stronger than the traditional labelers (i.e., we can simulate traditional active learning using such oracles), and is generally hard to deploy in practice. There turns out to be little general analysis on using a ''weaker'' interaction between human and computer. Balcan and Hanneke\cite{balcan2012robust} studied an abstract query based notions from exact learning, but their analysis cannot handle queries that gives relation between samples (as comparisons do). Our work fits in this blank.

We compare our work to traditional label-based active learning \cite{hanneke2014theory}, which has drawn a lot of attention in the society in recent years. Disagreement-based active learning has been shown to reach a near-optimal rate on classification problems \cite{hanneke2009adaptive}. Another line of research is margin-based active learning \cite{awasthi2016learning}, which aims at computational efficiency of learning halfspaces, under the large-margin assumption.


\input{Exp_Tsycomp.tex}
\section{Proof of Theorem \ref{thm:adgactsy}}
\input{proof_adgactsy.tex}	
\section{Proof of Theorem \ref{thm:adgacagn}}
\input{proof_adgacagn.tex}	
\section{Proof for A$^2$-ADGAC}
\input{proof_a2adgac.tex}
\section{Proof for Margin-ADGAC}
\input{proof_localadgac.tex}
\section{Proof of Lower Bounds}
\label{sec:prooflowerbound}
\input{proof_lowerbound.tex}
\end{document}

%% file: Introduction.tex
\pad
\section{Introduction}
\spad
Given high costs of obtaining labels for big datasets, interactive learning is gaining popularity in both practice and theory of machine learning. On the practical side, there has been an increasing interest in designing algorithms capable of engaging domain experts in two-way queries to facilitate more accurate and more effort-efficient learning systems (c.f. \cite{maji2014part,wah2014similarity}). On the theoretical side, study of interactive learning has led to significant advances such as exponential improvement of query complexity over passive learning under certain conditions (c.f. \cite{awasthi2016learning,awasthi2017power,balcan2006agnostic,castro2008minimax,hanneke2014theory,sabato2016interactive}). While most of these approaches to interactive learning fix the form of an oracle, e.g., the labeling oracle, and explore the best way of querying, recent work allows for multiple diverse forms of oracles~\cite{balcan2016noise,beygelzimer2016search,dekel2012selective,yang2009cost}.
The focus of this paper is on this latter setting, also known as active dual supervision \cite{attenberg2010unified}. We investigate how to recover a hypothesis $h$ that is a good approximator of the optimal classifier $h^*$, in terms of expected 0/1 error $\Pr_{X}[h(X)\ne h^*(X)]$, given limited access to labels on individual instances $X\in \mathcal{X}$ and pairwise comparisons about which one of two given instances is more likely to belong to the +1/-1 class.

Our study is motivated by important applications where comparisons are easier to obtain than labels, and the algorithm can leverage both types of oracles to improve label and total query complexity. 
For example, in material design, synthesizing materials for specific conditions requires expensive experimentation, but with an appropriate algorithm we can leverage expertize of material scientists, for whom it may be hard to accurately assess the resulting material properties, but who can quickly compare different input conditions and suggest which ones are more promising. 
Similarly, in clinical settings, precise assessment of each individual patient's health status can be difficult, expensive and/or risky (e.g.\ it may require application of invasive sensors or diagnostic surgeries), 
but comparing relative statuses of two patients at a time may be relatively easy and accurate. 
In both these scenarios we may have access to a modest amount of individually labeled data, but the bulk of more accessible training information is available via pairwise comparisons.
There are many other examples where humans find it easier to perform pairwise comparisons rather than providing direct labels, including content search~\cite{furnkranz2010preference}, image retrieval~\cite{wah2014similarity}, ranking~\cite{heckel2016active}, etc.

Despite many successful applications of comparison oracles, many fundamental questions remain. One of them is how to design \emph{noise-tolerant}, \emph{cost-efficient} algorithms that can approximate the unknown target hypothesis to arbitrary accuracy while having access to pairwise comparisons. On one hand, while there is theoretical analysis on the pairwise comparisons concerning the task of learning to rank \cite{ailon2007efficient,jamieson2011active}, estimating ordinal measurement models \cite{shah2014better} and learning combinatorial functions \cite{balcan2016learning}, much remains unknown how to extend these results to more generic hypothesis classes. On the other hand, although we have seen great progress on using single or multiple oracles with the same form of interaction  ~\cite{balcan2012robust,dekel2012selective}, classification using both comparison and labeling queries remains an interesting open problem.
Independently of our work, Kane et al.~\cite{kane2017active} concurrently analyzed a similar setting of learning to classify using both label and comparison queries. However, their algorithms work only in the noise-free setting.

\noindent{\textbf{Our Contributions:}}
Our work addresses the aforementioned issues by presenting a new algorithm, Active Data Generation with Adversarial Comparisons (ADGAC), which learns a classifier with both noisy labeling and noisy comparison oracles.
\spad
\begin{itemize}
\item We analyze ADGAC under Tsybakov (TNC) \cite{tsybakov2004optimal} and adversarial noise conditions for the labeling oracle, along with the adversarial noise condition for the comparison oracle. Our general framework can augment any active learning algorithm by replacing the batch sampling in these algorithms with ADGAC. Figure \ref{fig:flow} presents the work flow of our framework.
\item
We propose A$^2$-ADGAC algorithm, which can learn an arbitrary hypothesis class. The label complexity of the algorithm is as small as learning a threshold function under both TNC and adversarial noise condition, independently of the structure of the hypothesis class. The \emph{total query complexity} improves over previous best-known results under TNC which can only access the labeling oracle. 
\item
We derive Margin-ADGAC to learn the class of halfspaces. This algorithm has the same label and total query complexity as A$^2$-ADGAC, but is computationally efficient.
\item
We present lower bounds on total query complexity for any algorithm that can access both labeling and comparison oracles, and a noise tolerance lower bound for our algorithms. These lower bounds demonstrate that our analysis is nearly optimal.
\end{itemize}
\spad

\begin{figure}

\centering
\includegraphics[width=6in]{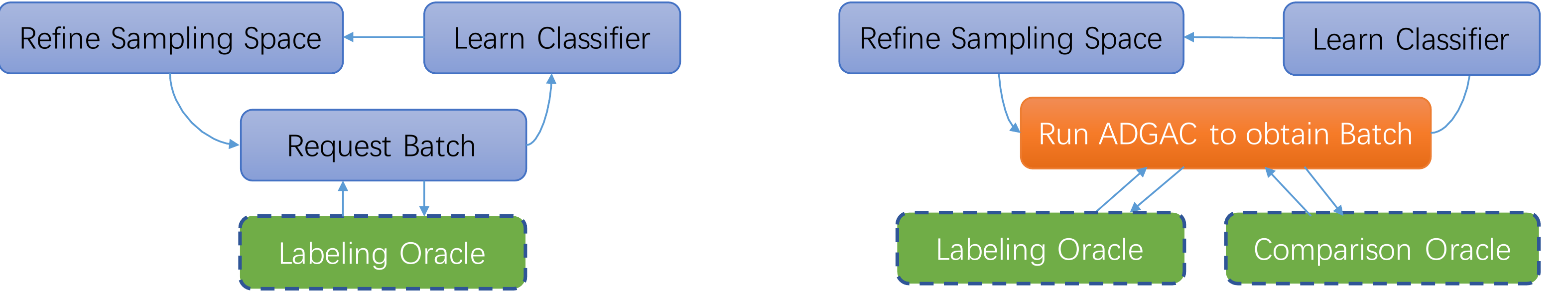}
\caption{Explanation of work flow of ADGAC-based algorithms. \textbf{Left:} Procedure of typical active learning algorithms. \textbf{Right:} Procedure of our proposed ADGAC-based interactive learning algorithm which has access to both pairwise comparison and labeling oracles.}\label{fig:flow}
\end{figure}


{\small \begin{table}
	\caption{Comparison of various methods for learning of generic hypothesis class (Omitting $\log(1/\varepsilon)$ factors).}
	\label{tab:genHyp}
	\centering
	\begin{tabular}{ccccc}
		\toprule
		Label Noise     & Work    & \# Label & \# Query & $\Tol_{\comp}$ \\
		\midrule
		Tsybakov ($\kappa$)  & \comment{Hanneke} \cite{hanneke2009adaptive} & $\tO\left(\left(\frac{1}{\varepsilon}\right)^{2\kappa-2}d\theta\right) $ & $\tO\left(\left(\frac{1}{\varepsilon}\right)^{2\kappa-2}d\theta\right) $ & N/A\\
		Tsybakov ($\kappa$)& \textbf{Ours} & $\tO\left(\left(\frac{1}{\varepsilon}\right)^{2\kappa-2}\right) $ & $\tO\left(\left(\frac{1}{\varepsilon}\right)^{2\kappa-2}\theta+d\theta\right)$ & $\O(\varepsilon^{2\kappa})$\\
		Adversarial ($\nu=\O(\varepsilon)$) & \comment{Hanneke} \cite{hanneke2014theory} & $\tO(d\theta)$ & $\tO(d\theta)$ & N/A \\
		Adversarial ($\nu=\O(\varepsilon)$)& \textbf{Ours} & $\tO(1)$ & $\tO(d\theta)$ & $\O(\varepsilon^2)$\\
		\bottomrule
	\end{tabular}
    \pad
    \spad
\end{table}}
An important quantity governing the performance of our algorithms is the adversarial noise level of comparisons: denote by $\Tol_{\comp}(\varepsilon,\delta,\mathcal{A})$ the adversarial noise tolerance level of comparisons that guarantees an algorithm $\mathcal{A}$ to achieve an error of $\varepsilon$, with probability at least $1-\delta$. Table \ref{tab:genHyp} compares our results with previous work in terms of label complexity, total query complexity, and $\Tol_{\comp}$ for generic hypothesis class $\bC$ with error $\varepsilon$. We see that our results significantly improve over prior work with the extra comparison oracle.
Denote by $d$ the VC-dimension of $\bC$ and $\theta$ the disagreement coefficient. 
We also compare the results in Table \ref{tab:learnhalf} for learning halfspaces under isotropic log-concave distributions.
In both cases, our algorithms enjoy small label complexity that is independent of $\theta$ and $d$. This is helpful when labels are very expensive to obtain. Our algorithms also enjoy better total query complexity under both TNC and adversarial noise condition for efficiently learning halfspaces.

\begin{table}
	\caption{Comparison of various methods for learning of halfspaces (Omitting $\log(1/\varepsilon)$ factors).}
	\label{tab:learnhalf}
	\centering
	\begin{tabular}{cccccc}
		\toprule
		Label Noise     & Work     & \# Label & \# Query  & $\Tol_{\comp}$& Efficient? \\
        \midrule
		Massart & \comment{Balcan et al.} \cite{balcan2007margin} & $\tO(d)$ & $\tO(d)$ & N/A& No\\
		Massart & \comment{Awasthi et al.} \cite{awasthi2016learning} & poly$(d)$ & poly$(d)$ & N/A& Yes\\
		Massart & \textbf{Ours} & $\tO(1)$ & $\tO(d)$ & $\O(\varepsilon^{2})$& Yes\\
		Tsybakov $(\kappa)$ & \comment{Hanneke} \cite{hanneke2014theory} & $\tO(\left(\frac{1}{\varepsilon}\right)^{2\kappa-2}d\theta)$ & $\tO(\left(\frac{1}{\varepsilon}\right)^{2\kappa-2}d\theta)$ & N/A& No\\
		Tsybakov $(\kappa)$& \textbf{Ours} & $\tO\left(\left(\frac{1}{\varepsilon}\right)^{2\kappa-2}\right)$ & $\tO\left(\left(\frac{1}{\varepsilon}\right)^{2\kappa-2}+d\right)$ & $\O(\varepsilon^{2\kappa})$& Yes \\
		Adversarial ($\nu=\O(\varepsilon)$) & \comment{Zhang et al.} \cite{zhang2014beyond} & $\tO(d)$ & $\tO(d)$ & N/A& No \\
		Adversarial ($\nu=\O(\varepsilon)$) & \comment{Awasthi et al.} \cite{awasthi2017power} & $\tO(d^2)$ & $\tO(d^2)$ & N/A& Yes \\
		Adversarial ($\nu=\O(\varepsilon)$) & \textbf{Ours} & $\tO(1)$ & $\tO(d)$ & $\O(\varepsilon^{2})$& Yes\\
		\bottomrule
	\end{tabular}
    \pad
    \spad
\end{table}


%% file: Setup.tex
\spad
\section{Preliminaries}
\spad
\noindent{\textbf{Notations:}}
We study the problem of learning a classifier $h: \mathcal{X}\rightarrow \mathcal{Y}=\{-1,1\}$, where $\mathcal{X}$ and $\mathcal{Y}$ are the instance space and label space, respectively. Denote by $\P_{\mathcal{XY}}$ the distribution over $\mathcal{X}\times \mathcal{Y}$ and let $\P_\mathcal{X}$ be the marginal distribution over $\mathcal{X}$. A hypothesis class $\bC$ is a set of functions $h: \mathcal{X}\rightarrow \mathcal{Y}$. 
For any function $h$, define the error of $h$ under distribution $D$ over $\mathcal{X}\times \mathcal{Y}$ as $\err_D(h)=\Pr_{(X,Y)\sim D}[h(X)\ne Y]$. Let $\err(h)=\err_{P_{\mathcal{XY}}}(h)$. Suppose that $h^*\in \bC$ satisfies $\err(h^*)=\inf_{h\in \bC} \err(h)$. For simplicity, we assume that such an $h^*$ exists in class $\mathbb C$.

We apply the concept of disagreement coefficient from Hanneke \cite{hanneke2009adaptive} for generic hypothesis class in this paper. In particular, for any set $V\subseteq \bC$, we denote by $\DIS(V)=\{x\in \mathcal{X}:\exists h_1,h_2\in V, h_1(x)\ne h_2(x) \}$. The disagreement coefficient is defined as $\theta=\sup_{r>0} \frac{\Pr[\DIS(B(h^*,r))]}{r}$, where $B(h^*,r)=\{h\in \bC: \Pr_{X\sim\P_\mathcal{X}}[h(X)\ne h^*(X)]\leq r\}$.

\noindent{\textbf{Problem Setup:}
We analyze two kinds of noise conditions for the labeling oracle, namely, adversarial noise condition and Tsybakov noise condition (TNC). We formally define them as follows.
\begin{condition}[Adversarial Noise Condition for Labeling Oracle]\label{cond:labeladv}
Distribution $\P_{\mathcal{XY}}$ satisfies adversarial noise condition for labeling oracle with parameter $\nu\geq 0$, if $ \nu=\Pr_{(X,Y)\sim\P_{\mathcal{XY}}}[Y\ne h^*(X)]$.
\end{condition}

\spad
\begin{condition}[Tsybakov Noise Condition for Labeling Oracle]
	\label{cond:labeltsy}
	Distribution $\P_{\mathcal{XY}}$ satisfies Tsybakov noise condition for labeling oracle with parameters $\kappa\geq 1,\mu\geq 0$, if $\forall h\in \mathbb{C},\ \err(h)-\err(h^*)\geq \mu\Pr_{X\sim\P_{\mathcal{X}}}[h(X)\ne h^*(X)]^\kappa$.
    The special case of $\kappa=1$ is also called Massart noise condition.
\end{condition}
\spad
For TNC, we assume that the above-defined $h^*$ is the Bayes optimal classifier, i.e., $h^*(x)=\sign(\eta(x)-1/2)$~\cite{boucheron2005theory,hanneke2009adaptive},\footnote{The assumption that $h^*$ is Bayes optimal classifier can be relaxed if the approximation error of $h^*$ can be quantified under assumptions on the decision boundary (c.f. \cite{castro2008minimax}).} where $\eta(x)=\Pr[Y=1|X=x]$.
In the classic active learning scenario, the algorithm has access to an unlabeled pool drawn from $\P_{\mathcal{X}}$. The algorithm can then query the labeling oracle for any instance from the pool. The goal is to find an $h\in \bC$ such that the error $\Pr[h(X)\ne h^*(X)]\leq \varepsilon$. The labeling oracle has access to the input $x\in \mathcal{X}$, and outputs $y\in \{-1,1\}$ according to $\P_{\mathcal{XY}}$. In our setting, however, an extra comparison oracle is available. This oracle takes as input a pair of instances $(x,x')\in \mathcal{X}\times\mathcal{X}$, and returns a variable $Z(x,x')\in \{-1,1\}$, where $Z(x,x')=1$ indicates that $x$ is more likely to be positive, while $Z(x,x')=-1$ otherwise. In this paper, we discuss an adversarial noise condition for the comparison oracle. We discuss about dealing with TNC on the comparison oracle in appendix.

\begin{condition}[Adversarial Noise Condition for Comparison Oracle]\label{cond:compadv}
Distribution $\P_{\mathcal{XXZ}}$
satisfies adversarial noise with parameter $\nu'\geq 0$, if 
$\nu'=\Pr[Z(X,X')(h^*(X)-h^*(X'))<0]$. 
\end{condition}
\comment{
\spad
Note that the noise is defined only for labels $h^*(X)$ and $h^*(X')$; the oracle does not incur any error for $h^*(X)=h^*(X')$.
\spad}
\spad

For an interactive learning algorithm $\mathcal{A}$, given error $\varepsilon$ and failure probability $\delta$, let $\SC_{\comp}(\varepsilon,\delta,\mathcal{A})$ and $\SC_{\lab}(\varepsilon,\delta,\mathcal{A})$ be the comparison and label complexity, respectively. The query complexity of $\mathcal{A}$ is defined as the sum of label and comparison complexity. Similar to the definition of $\Tol_{\comp}(\varepsilon,\delta,\mathcal{A})$,  define $\Tol_{\lab}(\varepsilon,\delta,\mathcal{A})$ as the maximum $\nu$ such that algorithm $\mathcal{A}$ achieves an error of at most $\varepsilon$ with probability $1-\delta$. As a summary, $\mathcal{A}$ learns an $h$ such that $\Pr[h(X)\ne h^*(X)]\leq \varepsilon$ with probability $1-\delta$ using $\SC_{\comp}(\varepsilon,\delta,\mathcal{A})$ comparisons and $\SC_{\lab}(\varepsilon,\delta,\mathcal{A})$ labels, if $\nu\leq \Tol_{\lab}(\varepsilon,\delta,\mathcal{A})$ and $\nu'\leq \Tol_{\comp}(\varepsilon,\delta,\mathcal{A})$.
We omit the parameters of $\SC_{\comp},\SC_{\lab},\Tol_{\comp},\Tol_{\lab}$ if they are clear from the context. We use $\O(\cdot)$ to express sample complexity and noise tolerance, and $\tilde{O}(\cdot)$ to ignore the $\log(\cdot)$ terms. Table \ref{tab:notation} summarizes the main notations throughout the paper.

\begin{table}[t]
	\caption{Summary of notations.}
	\label{tab:notation}
	\centering
	\begin{tabular}{cc|cc}
		\toprule
		Notation     & Meaning  & Notation & Meaning\\
		\midrule
		$\bC$ & Hypothesis class & $\kappa$ & Tsybakov noise level (labeling)\\
        $X,\mathcal{X}$ & Instance \& Instance space & $\nu$ & Adversarial noise level (labeling) \\
        $Y,\mathcal{Y}$ & Label \& Label space & $\nu'$ & Adversarial noise level (comparison)\\
        $Z,\mathcal{Z}$ & Comparison \& Comparison space & $\err_D(h)$ & Error of $h$ on distribution $D$ \\
        $d$ & VC dimension of $\bC$ & $\SC_{\lab}$ & Label complexity \\
        $\theta$ & Disagreement coefficient & $\SC_{\comp}$ & Comparison complexity \\
        $h^*$ & Optimal classifier in $\bC$ & $\Tol_{\lab}$ & Noise tolerance (labeling)\\
        $g^*$ & Optimal scoring function &$\Tol_{\comp}$ & Noise tolerance (comparison)\\
		\bottomrule
	\end{tabular}
    \pad
    \spad
\end{table}

%% file: adgac.tex
\spad
\section{Active Data Generation with Adversarial Comparisons (ADGAC)}
\spad
The hardness of learning from pairwise comparisons follows from the error of comparison oracle: the comparisons are \emph{noisy}, and can be \emph{asymmetric} and \emph{intransitive}, meaning that the human might give contradicting preferences like $x_1\preccurlyeq x_2 \preccurlyeq x_1$ or $x_1\preccurlyeq x_2\preccurlyeq x_3 \preccurlyeq x_1$ (here $\preccurlyeq$ is some preference).
This makes traditional methods, e.g., defining a function class $\{h: h(x)=Z(x,\hat{x}), \hat{x}\in \mathcal{X} \}$, fail, because such a class may have infinite VC dimension.

In this section, we propose a novel algorithm, ADGAC, to address this issue. Having access to both comparison and labeling oracles, ADGAC generates a labeled dataset by techniques inspired from group-based binary search. We show that ADGAC can be combined with any active learning procedure to obtain  interactive algorithms that can utilize both labeling and comparison oracles. We provide theoretical guarantees for ADGAC.

\spad
\subsection{Algorithm Description}
\spad
To illustrate ADGAC, we start with a general active learning framework in Algorithm~\ref{algo:alframe}. Many active learning algorithms can be adapted to this framework, such as A$^2$ \cite{balcan2006agnostic} and margin-based active algorithms~\cite{awasthi2017power,awasthi2016learning}. Here $U$ represents the querying space/disagreement region of the algorithm (i.e., we reject an instance $x$ if $x\not\in U$), and $V$ represents a version space consisting of potential classifiers. 
For example, A$^2$ algorithm can be adapted to Algorithm \ref{algo:alframe} straightforwardly by keeping $U$ as the sample space and $V$ as the version space. More concretely, A$^2$ algorithm \cite{balcan2006agnostic} for adversarial noise can be characterized by
\[U_0=\mathcal{X},\ V_0=\bC,\ f_V(U,V,W,i)=\{h:|W|\err_W(h)\leq n_i\varepsilon_i\},\ f_U(U,V,W,i)=\DIS(V), \]
where $\varepsilon_i$ and $n_i$ are parameters of the A$^2$ algorithm, and $\DIS(V)=\{x\in \mathcal{X}:\exists h_1,h_2\in V, h_1(x)\ne h_2(x) \}$ is the disagreement region of $V$.
Margin-based active learning \cite{awasthi2017power} can also be fitted into Algorithm \ref{algo:alframe} by taking $V$ as the halfspace that (approximately) minimizes the hinge loss, and $U$ as the region within the margin of that halfspace.

\begin{algorithm}
	\caption{Active Learning Framework}
	\label{algo:alframe}
	\begin{algorithmic}[1]		
		\Require{$\varepsilon,\delta$, a sequence of $n_i$, functions $f_U,f_V$.}
		\State Initialize $U\leftarrow U_0\subseteq \mathcal{X}, V\leftarrow V_0\subseteq \bC$.
		\For{$i=1,2,...,\log(1/\varepsilon)$}
		\State Sample unlabeled dataset $\tilde{S}$ of size $n_i$. Let $S\leftarrow \{x:x\in \tilde S, x\in U\}$.
        \State Request the labels of $x\in S$ and obtain $W\leftarrow \{(x_i,y_i):x_i\in S\}$. \label{step:sample}
		\State Update $V\leftarrow f_V(U,V,W,i)$, $U\leftarrow f_U(U,V,W,i)$.
		\EndFor
        \Ensure{Any classifier $\hat{h}\in V$.}
	\end{algorithmic}
\end{algorithm}

To efficiently apply the comparison oracle, we propose to replace step \ref{step:sample} in Algorithm \ref{algo:alframe} with a subroutine, ADGAC, that has access to both comparison and labeling oracles. Subroutine \ref{algo:adgac} describes ADGAC. It takes as input a dataset $S$ and a sampling number $k$. 
ADGAC first runs Quicksort algorithm on $S$ using feedback from comparison oracle, which is of form $Z(x,x')$. Given that the comparison oracle $Z(\cdot,\cdot)$ might be asymmetric w.r.t. its two arguments, i.e., $Z(x,x')$ may not equal to $Z(x',x)$, for each pair $(x_i,x_j)$, we randomly choose $(x_i,x_j)$ or $(x_j,x_i)$ as the input to $Z(\cdot,\cdot)$. After Quicksort, the algorithm divides the data into multiple groups of size $\alpha m=\varepsilon |\tilde{S}|$, and does group-based binary search by sampling $k$ labels from each group and determining the label of each group by majority vote.

For active learning algorithm $\mathcal{A}$, let $\mathcal{A}$-ADGAC be the algorithm of replacing step \ref{step:sample} with ADGAC using parameters $\left(S_i,n_i,\varepsilon_i,k_i\right)$, where $\varepsilon_i,k_i$ are chosen as additional parameters of the algorithm. We establish results for specific $\mathcal{A}$: A$^2$ and margin-based active learning in Sections \ref{sec:a2adgac} and \ref{sec:localadgac}, respectively.


\pad
\begin{algorithm}[t]
	\newalgname{Subroutine}
	\caption{Active Data Generation with Adversarial Comparison (ADGAC)}
	\label{algo:adgac}
	\begin{algorithmic}[1]		
		\Require{Dataset $S$ with $|S|=m$, $n$, $\varepsilon,k$.}
		\State $\alpha\leftarrow \frac{\varepsilon n}{2m}$.
		\State Define preference relation on $S$ according to $Z$. Run Quicksort on $S$ to rank elements in an increasing order. 
		Obtain a sorted list $S=(x_1,x_2,...,x_m)$.\label{step:quicksort}\label{step:rank}
		\State Divide $S$ into groups of size $\alpha m$: $S_i=\{x_{(i-1)\alpha m+1},...,x_{i\alpha m}\}, i=1,2,...,1/\alpha$ \label{step:divide}.
		\State $t_{\min}\leftarrow 1, t_{\max} \leftarrow 1/\alpha$.
		\While{$t_{\min}<t_{\max}$}\Comment{Do binary search}
		\State $t=(t_{\min}+t_{\max})/2$.
		\State Sample $k$ points uniformly without replacement from $S_t$ and obtain the labels $Y=\{y_1,...,y_k\}$.
        \State \textbf{If } $\sum_{i=1}^k y_i\geq 0$, \textbf{ then }$t_{\max}=t$; \textbf{ else } $t_{\min}=t+1$.
		\EndWhile
		\State For $t'>t$ and $x_i\in S_{t'}$, let $\hat{y}_i\leftarrow 1$. \label{step:havet}
		\State For $t'<t$ and $x_i\in S_{t'}$, let $\hat{y}_i\leftarrow -1$.
		\State For $x_i\in S_{t}$, let $\hat{y}_i$ be the majority of labeled points in $S_t$.
        \Ensure{Predicted labels $\hat{y}_1,\hat{y}_2,...,\hat{y}_m$.}
	\end{algorithmic}
\end{algorithm}

\subsection{Theoretical Analysis of ADGAC}

Before we combine ADGAC with active learning algorithms, we provide  theoretical results for ADGAC. By the algorithmic procedure, ADGAC reduces the problem of labeling the whole dataset $S$ to binary searching a threshold on the sorted list $S$. One can show that the conflicting instances cannot be too many within each group $S_i$, and thus binary search performs well in our algorithm. We also use results in \cite{ailon2007efficient} to give an error estimate of Quicksort. We have the following result based on the above arguments.
\begin{theorem}\label{thm:adgactsy}
	Suppose that Conditions \ref{cond:labeltsy} and \ref{cond:compadv} hold for $\kappa\geq 1, \nu'\geq 0$, and  $n=\Omega\left(\left(\frac{1}{\varepsilon}\right)^{2\kappa-1}\log(1/\delta) \right)$. Assume a set $\tilde{S}$ with $|\tilde{S}|=n$ is sampled i.i.d. from $\P_\mathcal{X}$ and $S\subseteq \tilde{S}$ is an arbitrary subset of $\tilde{S}$ with $|S|=m$.
	There exist absolute constants $C_1,C_2,C_3$ such that if we run Subroutine $\ref{algo:adgac}$ with $\varepsilon<C_1$, $\nu'\leq C_2\varepsilon^{2\kappa}\delta$, 
     $k=k^{(1)}(\varepsilon,\delta)\coloneqq C_3\log\left(\frac{\log(1/\varepsilon)}{\delta}\right)\left(\frac{1}{\varepsilon}\right)^{2\kappa-2}$,
it will output a labeling of $S$ such that $|\{x_i\in S:\hat{y}_i\ne h^*(x_i) \}|\leq \varepsilon n$, with probability at least $1-\delta$. The expected number of comparisons required is $\O(m\log m)$, and the number of sample-label pairs required is $\SC_{\lab}(\varepsilon,\delta)=\tilde{\O}\left(\log\left(\frac{m}{\varepsilon n}\right)\log(1/\delta)\left(\frac{1}{\varepsilon}\right)^{2\kappa-2}\right)$.
	
\end{theorem}
\pad
Similarly, we analyze ADGAC under adversarial noise condition w.r.t. labeling oracle with $\nu=\O(\varepsilon)$.
\begin{theorem}\label{thm:adgacagn}
	Suppose that Conditions \ref{cond:labeladv} and \ref{cond:compadv} hold for $\nu,\nu'\geq 0$, and $n=\Omega\left(\frac{1}{\varepsilon}\log(1/\delta) \right)$. Assume a set $\tilde{S}$ with $|\tilde{S}|=n$ is sampled i.i.d. from $\P_\mathcal{X}$ and $S\subseteq \tilde{S}$ is an arbitrary subset of $\tilde{S}$ with $|S|=m$.
    There exist absolute constants $C_1,C_2,C_3,C_4$ such that if we run Subroutine $\ref{algo:adgac}$ with $\varepsilon<C_1$, $\nu'\leq C_2\varepsilon^{2}\delta$, 
    $k=k^{(2)}(\varepsilon,\delta)\coloneqq C_3\log\left(\frac{\log(1/\varepsilon)}{\delta}\right)$,
    and $\nu\leq C_4\varepsilon$, 
it will output a labeling of $S$ such that $|\{x_i\in S:\hat{y}_i\ne h^*(x_i) \}|\leq \varepsilon n$, with probability at least $1-\delta$.
    The expected number of comparisons required is $\O(m\log m)$, and the number of sample-label pairs required is $\SC_{\lab}(\varepsilon,\delta)=\O\left(\log\left(\frac{m}{\varepsilon n}\right)\log\left(\frac{\log(1/\varepsilon)}{\delta}\right)\right)$.
	
\end{theorem}
\pad

Theorems \ref{thm:adgactsy} and \ref{thm:adgacagn} show that ADGAC gives a labeling of dataset with arbitrary small error using label complexity \emph{independent} of the data size. Moreover, ADGAC is computationally efficient and distribution-free.
These nice properties of ADGAC lead to improved query complexity when we combine ADGAC with other active learning algorithms.


%% file: General.tex
\pad
\section{A$^2$-ADGAC: Learning of Generic Hypothesis Class}
\label{sec:a2adgac}
\spad
In this section, we combine ADGAC with A$^2$ algorithm to learn a generic hypothesis class. We use the framework in Algorithm \ref{algo:alframe}:
let A$^2$-ADGAC be the algorithm that replaces step \ref{step:sample} in Algorithm \ref{algo:alframe} with ADGAC of parameters $\left(S, n_i, \varepsilon_i, k_i\right)$, where $n_i,\varepsilon_i,k_i$ are parameters to be specified later.
Under TNC, we have the following result.
\begin{theorem}\label{thm:a2adgactsy}
	Suppose that Conditions \ref{cond:labeltsy} and \ref{cond:compadv} hold, and $h^*(x)=\sign(\eta(x)-1/2)$.  There exist global constants $C_1,C_2$ such that if we run A$^2$-ADGAC with $\varepsilon<C_1,\delta$, $\nu'\leq \Tol_{\comp}(\varepsilon,\delta) = C_2\varepsilon^{2\kappa}\delta$, $\varepsilon_i=2^{-(i+2)}, n_i= \Omega\left(\frac{1}{\varepsilon_i}\left(d\log(1/\varepsilon)\right)+\left(\frac{1}{\varepsilon_i}\right)^{2\kappa-1}\log(1/\delta)\right)$, $ k_i=k^{(1)}\left(\varepsilon_i,\frac{\delta}{4\log(1/\varepsilon)}\right)$ with $k^{(1)}$ specified in Theorem \ref{thm:adgactsy}, with probability at least $1-\delta$, the algorithm will return a classifier $\hat{h}$ with $\Pr[\hat{h}(X)\ne h^*({X}) ]\leq \varepsilon$ with comparison and label complexity
    \spad
	\[\E[\SC_{\comp}]=\tilde{\O}\left(\theta\log^2\left(\frac{1}{\varepsilon}\right)\log (d\theta)\left(\left(d\log\left(\frac{1}{\varepsilon}\right)\right)+\left(\frac{1}{\varepsilon}\right)^{2\kappa-2}\log(1/\delta)\right)\right),\] 
    \spad
    \[\SC_{\lab}=\tilde{\O}\left(\log\left(\frac{1}{\varepsilon}\right)\log\left(\min\left\{\frac{1}{\varepsilon},\theta\right\}\right)\log(1/\delta)\left(\frac{1}{\varepsilon}\right)^{2\kappa-2}\right). \]
	The dependence on $\log^2(1/\varepsilon)$ in $\SC_{\comp}$ can be reduced to $\log(1/\varepsilon)$ under Massart noise.
\end{theorem}

We can prove a similar result for adversarial noise condition.
\begin{theorem}\label{thm:a2adgacagn}
	Suppose that Conditions \ref{cond:labeladv} and \ref{cond:compadv} hold. There exist global constants $C_1,C_2,C_3$ such that if we run A$^2$-ADGAC with $\varepsilon<C_1,\delta$, $\nu'\leq \Tol_{\comp}(\varepsilon,\delta) =C_2\varepsilon^2\delta, \nu\leq \Tol_{\lab}(\varepsilon,\delta) = C_3\varepsilon$, $\varepsilon_i=2^{-(i+2)}, n_i= \tilde{\Omega}\left(\frac{1}{\varepsilon_i}d\log\left(\frac{1}{\varepsilon_i}\right)\log(1/\delta)\right),  k_i=k^{(2)}\left(\varepsilon_i,\frac{\delta}{4\log(1/\varepsilon)}\right)$ with $k^{(2)}$ specified in Theorem \ref{thm:adgacagn}, with probability at least $1-\delta$, the algorithm will return a classifier $\hat{h}$ with $\Pr[\hat{h}(X)\ne h^*({X}) ]\leq \varepsilon$ with comparison and label complexity
\spad
\[\E[\SC_{\comp}]=\tilde{\O}\left(\theta d \log(\theta d) \log\left(\frac{1}{\varepsilon_i}\right)\log(1/\delta)\right),\]
    \spad
    \[\SC_{\lab}=\tilde{\O}\left(\log\left(\frac{1}{\varepsilon}\right)\log\left(\min\left\{\frac{1}{\varepsilon},\theta\right\}\right)\log(1/\delta)\right). \]
\end{theorem}

Theorems \ref{thm:a2adgactsy} and \ref{thm:a2adgacagn} show that having access to even a biased comparison function can reduce the problem of learning a classifier in high-dimensional space to that of learning a threshold classifier in one-dimensional space as the label complexity matches that of actively learning a threshold classifier. 
Given the fact that comparisons are usually easier to obtain, A$^2$-ADGAC will save a lot in practice due to its small label complexity.
More importantly, we improve the total query complexity under TNC by separating the dependence on $d$ and $\varepsilon$; The query complexity is now the sum of the two terms instead of the product of them. This observation shows the power of pairwise comparisons for learning classifiers. Such small label/query complexity is impossible without access to a comparison oracle, since query complexity with only labeling oracle is at least $\Omega\left(d\left(\frac{1}{\varepsilon}\right)^{2\kappa-2}\right)$ and $\Omega\left(d\log\left(\frac{1}{\varepsilon}\right)\right)$ under TNC and adversarial noise conditions, respectively \cite{hanneke2014theory}. Our results also matches the lower bound of learning with labeling and comparison oracles up to log factors (see Section \ref{sec:lowerbound}).

We note that Theorems \ref{thm:a2adgactsy} and \ref{thm:a2adgacagn} require rather small $\Tol_{\comp}$, equal to $\O(\varepsilon^{2\kappa}\delta)$ and $\O(\varepsilon^{2}\delta)$, respectively. We will show in Section \ref{sec:upperboundnup} 
that it is necessary to require $\Tol_{\comp}=\O(\varepsilon^{2})$ in order to obtain a classifier of error $\varepsilon$, if we restrict the use of labeling oracle to only learning a threshold function. Such restriction is able to reach the near-optimal label complexity as specified in Theorems \ref{thm:a2adgactsy} and \ref{thm:a2adgacagn}.

%% file: Halfspace.tex
\pad

\section{Margin-ADGAC: Learning of Halfspaces}
\label{sec:localadgac}
\spad
In this section, we combine ADGAC with margin-based active learning  \cite{awasthi2017power} to efficiently learn the class of halfspaces. Before proceeding, we first mention a naive idea of utilizing comparisons: we can i.i.d. sample pairs $(x_1,x_2)$ from $\P_\mathcal{X}\times \P_\mathcal{X}$, and use $Z(x_1,x_2)$ as the label of $x_1-x_2$, where $Z$ is the feedback from comparison oracle. However, this method cannot work well in our setting without additional assumption on the noise condition for the labeling $Z(x_1,x_2)$.


Before proceeding, we assume that $\P_{\mathcal{X}}$ is isotropic log-concave on $\mathbb{R}^d$; i.e., $\P_{\mathcal{X}}$ has mean 0, covariance $I$ and the logarithm of its density function is a concave function~\cite{awasthi2016learning,awasthi2017power}. The hypothesis class of halfspaces can be represented as $\bC=\{h:h(x)=\sign(w\cdot x), w\in \mathbb{R}^d \}$. Denote by $h^*(x)=\sign(w^*\cdot x)$ for some $w^*\in \mathbb{R}^d$. Define $l_{\tau}(w,x,y)=\max\left(1-y(w\cdot x)/\tau,0\right)$ and $l_\tau(w,W)=\frac{1}{|W|}\sum_{(x,y)\in W} l_\tau(w,x,y)$ as the hinge loss. The expected hinge loss of $w$ is $L_\tau(w,D)=\E_{x\sim D}[l_\tau(w,x,\sign(w^*\cdot x))]$. 


Margin-based active learning~ \cite{awasthi2017power} is a concrete example of Algorithm \ref{algo:alframe} by taking $V$ as (a singleton set of) the hinge loss minimizer, while taking $U$ as the margin region around that minimizer. More concretely, take $U_0=\mathcal{X}$ and $V_0=\{w_0\}$ for some $w_0$ such that $\theta(w_0,w^*)\leq \pi/2$. The algorithm works with constants $M\geq 2,\kappa<1/2$ and a set of parameters $r_i,\tau_i,b_i,z_i$ that equal to $\Theta(M^{-i})$. $V$ always contains a single hypothesis. Suppose $V=\{w_{i-1}\}$ in iteration $i-1$.  Let $v_i$ satisfies $l_{\tau_i}(v_i,W)\leq \min_{v:\|v-w_{i-1}\|_2\leq r_i, \|v\|_2\leq 1} l_{\tau_i}(v,W)+\kappa/8$, where $w_{i}$ is the content of $V$ in iteration $i$. We also have $f_V(V,W,i)=\{w_i\}=\left\{\frac{v_i}{\|v_i\|_2} \right\}$ and $f_U(U,V,W,i)=\{x: |w_i\cdot x|\leq b_i \}$.

Let Margin-ADGAC be the algorithm obtained by replacing the sampling step in margin-based active learning with ADGAC using parameters $\left(S, n_i, \varepsilon_i, k_i\right)$, where $n_i, \varepsilon_i, k_i$ are additional parameters to be specified later. We have the following results under TNC and adversarial noise conditions, respectively.
\begin{theorem}\label{thm:localadgactsy}
	Suppose that Conditions \ref{cond:labeltsy} and \ref{cond:compadv} hold, and $h^*(x)=\sign(w^*\cdot x)=\sign(\eta(x)-1/2)$. There are settings of $M, \kappa, r_i,\tau_i,b_i, \varepsilon_i,k_i$, and constants $C_1, C_2$ such that for all $\varepsilon\leq C_1, \nu'\leq \Tol_{\comp}(\varepsilon,\delta)= C_2\varepsilon^{2\kappa}\delta$, if we run Margin-ADGAC with $w_0$ such that $\theta(w_0,w^*)\leq \pi/2$, and  $n_i=\tilde{O}\left(\frac{1}{\varepsilon_i}d\log^3(dk/\delta)+\left(\frac{1}{\varepsilon}\right)^{2\kappa-1}\log(1/\delta)\right)$, it finds $\hat{w}$ such that $\Pr[\sign(\hat{w}\cdot X)\ne \sign(w^*\cdot X)]\leq \varepsilon$ with probability at least $1-\delta$. The comparison and label complexity are
	\[\E[\SC_{\comp}]=\tilde{\O}\left(\log^2(1/\varepsilon)\left(d\log^4(d/\delta)+\left(\frac{1}{\varepsilon}\right)^{2\kappa-2}\log(1/\delta)\right)\right),\]\spad
	\[\SC_{\lab}=\tilde{\O}\left(\log(1/\varepsilon)\log(1/\delta)\left(\frac{1}{\varepsilon}\right)^{2\kappa-2}\right).\]	
The dependence on $\log^2(1/\varepsilon)$ in $\SC_{\comp}$ can be reduced to $\log(1/\varepsilon)$ under Massart noise.
\end{theorem}
\begin{theorem}\label{thm:localadgacagn}
	Suppose that Conditions \ref{cond:labeladv} and \ref{cond:compadv} hold. There are settings of $M, \kappa, r_i,\tau_i,b_i, \varepsilon_i,k_i$, and constants $C_1, C_2, C_3$ such that for all $\varepsilon\leq C_1, \nu'\leq \Tol_{\comp}(\varepsilon,\delta)= C_2\varepsilon^{2\kappa}\delta, \nu\leq \Tol_{\comp}(\varepsilon,\delta)= C_3\varepsilon$, if we run Margin-ADGAC with $n_i=\tilde{\O}\left(\frac{1}{\varepsilon_i}d\log^3(dk/\delta)\right)$ and $w_0$ such that $\theta(w_0,w^*)\leq \pi/2$, it finds $\hat{w}$ such that $\Pr[\sign(\hat{w}\cdot X)\ne \sign(w^*\cdot X)]\leq \varepsilon$ with probability at least $1-\delta$. The comparison and label complexity are
     \sspad
	\[\E[\SC_{\comp}]=\tilde{\O}\left(\log(1/\varepsilon)\left(d\log^4(d/\delta)\right)\right),\quad\SC_{\lab}=\tilde{\O}\left(\log(1/\varepsilon)\log(1/\delta)\right).\]
\end{theorem}
\spad
The proofs of Theorems \ref{thm:localadgactsy} and \ref{thm:localadgacagn} are different from the conventional analysis of margin-based active learning in two aspects: a) Since we use labels generated by ADGAC, which is not independently sampled from the distribution $\P_{\mathcal{XY}}$, we require new techniques that can deal with adaptive noises; b) We improve the results of \cite{awasthi2017power} over the dependence of $d$ by new Rademacher analysis.

Theorems \ref{thm:localadgactsy} and \ref{thm:localadgacagn} enjoy better label and query complexity than previous results (see Table \ref{tab:learnhalf}). We mention that while Yan and Zhang~\cite{yan2017revisiting} proposed a perceptron-like algorithm with label complexity as small as $\tilde{\O}(d\log(1/\varepsilon))$ under Massart and adversarial noise conditions, their algorithm works only under uniform distributions over the instance space. In contrast, our algorithm Margin-ADGAC works under broad log-concave distributions.
The label and total query complexity of Margin-ADGAC improves over that of traditional active learning. The lower bounds in Section \ref{sec:lowerbound} show the optimality of our complexity.

%% file: LowerBounds.tex
\pad
\section{Lower Bounds}
\label{sec:lowerbound}
\pad
In this section, we give lower bounds on learning using labeling and pairwise comparison. In Section \ref{sec:labellowerbound}, we give a lower bound on the optimal label complexity $\SC_{\lab}$. In Section \ref{sec:querylowerbound} we use this result to give a lower bound on the total query complexity, i.e., the sum of comparison and label complexity. Our two methods match these lower bounds up to log factors. In Section \ref{sec:upperboundnup}, we additionally give an information-theoretic bound on $\Tol_{\comp}$, which matches our algorithms in the case of Massart and adversarial noise.

Following from \cite{hanneke2014theory,hanneke2012surrogate}, we assume that there is an underlying score function $g^*$ such that $h^*(x)=\sign(g^*(x))$. Note that $g^*$ does not necessarily have relation with $\eta(x)$; We only require that $g^*(x)$ represents how likely a given $x$ is positive. For instance, in digit recognition, $g^*(x)$ represents how an image looks like a 7 (or 9); In the clinical setting, $g^*(x)$ measures the health condition of a patient. Suppose that the distribution of $g^*(X)$ is continuous, i.e., the probability density function exists and for every $t\in \mathbb{R}$, $\Pr[g^*(X)=t]=0$.
\pad
\subsection{Lower Bound on Label Complexity}
\spad
\label{sec:labellowerbound}
The definition of $g^*$ naturally induces a comparison oracle $Z$ with $Z(x,x')=\sign(g^*(x)-g^*(x'))$. We note that this oracle is invariant to shifting w.r.t. $g^*$, i.e., $g^*$ and $g^*+t$ lead to the same comparison oracle. As a result, we cannot distinguish $g^*$ from $g^*+t$ without labels. 
In other words, pairwise comparisons do not help in improving label complexity when we are learning a threshold function on $\mathbb{R}$, where all instances are in the natural order. So the label complexity of any algorithm is lower bounded by that of learning a threshold classifier, and we formally prove this in the following theorem.
\spad
\begin{theorem}
\label{thm:labellowerbound}
For any algorithm $\mathcal{A}$ that can access both labeling and comparison oracles,  sufficiently small $\varepsilon,\delta$, and any score function $g$ that takes at least two values on $\mathcal{X}$, there exists a distribution $P_{\mathcal{XY}}$ satisfying Condition \ref{cond:labeltsy} such that the optimal function is in the form of $h^*(x)=\sign(g(x)+t)$ for some $t\in \mathbb{R}$ and
\spad
\begin{align}
\SC_{\lab} (\varepsilon,\delta,\mathcal{A})=\Omega\left(\left(1/\varepsilon\right)^{2\kappa-2}\log(1/\delta) \right). \label{eqn:labellowerbound}
\end{align} 
If $\P_{\mathcal{XY}}$ satisfies Condition \ref{cond:labeladv} with $\nu=O(\varepsilon)$, $\SC_{\lab}$ satisfies \eqref{eqn:labellowerbound} with $\kappa=1$.
\end{theorem}
\spad
The lower bound in Theorem \ref{thm:labellowerbound} matches the label complexity of A$^2$-ADGAC and Margin-ADGAC up to a log factor. So our algorithm is near-optimal.

\pad
\subsection{Lower Bound on Total Query Complexity}
\spad
\label{sec:querylowerbound}
We use Theorem \ref{thm:labellowerbound} to give lower bounds on the total query complexity of any algorithm which can access both comparison and labeling oracles.
\sspad
\begin{theorem}
	\label{thm:totallowerbound}
	For any algorithm $\mathcal{A}$ that can access both labeling and comparison oracles, and sufficiently small $\varepsilon,\delta$, there exists a distribution $P_{\mathcal{XY}}$ satisfying Condition \ref{cond:labeltsy},  such that
    \spad
\begin{equation}
\label{equ: lower bound equ}
\SC_{\comp}(\varepsilon,\delta,\mathcal{A})+\SC_{\lab}(\varepsilon,\delta,\mathcal{A})=\Omega\left(\left(1/\varepsilon\right)^{2\kappa-2}\log(1/\delta) +d\log(1/\varepsilon) \right).
\end{equation}
If $\P_{\mathcal{XY}}$ satisfies Condition \ref{cond:labeladv} with $\nu=O(\varepsilon)$, 
$\SC_{\comp}+\SC_{\lab}$ satisfies \eqref{equ: lower bound equ} with $\kappa=1$.
\end{theorem}
The first term of \eqref{equ: lower bound equ} follows from Theorem \ref{thm:labellowerbound}, whereas the second term follows from transforming a lower bound of active learning with access to only the labeling oracle. The lower bounds in Theorem \ref{thm:totallowerbound} match the performance of A$^2$-ADGAC and Margin-ADGAC up to log factors.
\comment{
We can also prove similar bounds in the adversarial case that matches the performance of A$^2$-ADGAC and Margin-ADGAC. This shows the near-optimality of our two algorithms.
}
\pad
\subsection{Adversarial Noise Tolerance of Comparisons}
\label{sec:upperboundnup}
\spad
Note that label queries are typically expensive in practice. Thus it is natural to ask the following question: what is the minimal requirement on $\nu'$, given that we are only allowed to have access to minimal label complexity as in Theorem \ref{thm:labellowerbound}? We study this problem in this section. More concretely, we study the requirement on $\nu'$ when we learn a threshold function using labels.
Suppose that the comparison oracle gives feedback using a scoring function $\hat{g}$, i.e., $Z(x,x')=\sign(\hat{g}(x)-\hat{g}(x'))$, and has error $\nu'$. We give a sharp minimax bound on the risk of the optimal classifier in the form of $h(x)=\sign(\hat{g}(x)-t)$ for some $t\in \mathbb{R}$ below. 
\begin{theorem}
	\label{thm:minimaxnup}
	Suppose that $\min\{\Pr[h^*(X)=1], \Pr[h^*(X)=-1]\}\geq \sqrt{\nu'}$  and both $\hat{g}(X)$ and $g^*(X)$ have probability density functions. If $\hat{g}(X)$ induces an oracle with error $\nu'$, then we have $\min_t\max_{\hat{g},g^*} \Pr[\sign(\hat{g}(X)-t)\ne h^*(X)]=\sqrt{\nu'}$.
    \spad
\end{theorem}

By Theorem \ref{thm:minimaxnup}, we see that the condition of $\nu'=\varepsilon^2$ is necessary if labels from $g^*$ are only used to learn a threshold on $\hat{g}$. This matches our choice of $\nu'$ under Massart and adversarial noise conditions for labeling oracle (up to a factor of $\delta$).

%

%% file: Conclusion.tex
\pad
\section{Conclusion}
\spad
We presented a general algorithmic framework, ADGAC, for learning with both comparison and labeling oracles. 
We proposed two variants of the base algorithm, A$^2$-ADGAC and Margin-ADGAC, to facilitate low query complexity under Tsybakov and adversarial noise conditions. The performance of our algorithms matches lower bounds for learning with both oracles. Our analysis is relevant to a wide range of practical applications where it is easier, less expensive, and/or less risky to obtain pairwise comparisons than  labels.

%% file: Exp_Tsycomp.tex
\section{Learning under TNC for Comparisons}\label{sec:tnccomp}
In this section we justify our choice of analyzing adversarial noise model for the comparison oracle. In fact, any algorithm using adversarial comparisons can be transformed into an algorithm using TNC comparisons, by treating learning comparison functions as a separate learning problem. Let $\mathbb{C}'$ be a hypothesis class consisting of comparison functions $f:\mathcal{X}\times\mathcal{X}\rightarrow \{-1,1\}$. Suppose the optimal comparison function is $f^*(x,x')=\sign(g^*(x)-g^*(x'))$, and Tsybakov noise condition holds for $((X,X'),Z)$ with some constant $\mu',\kappa'$; i.e., for any $f\in \mathbb{C}'$ we have
\[\Pr[f(X,X')\ne Z]-\Pr[f^*(X,X')\ne Z]\geq \mu'\Pr[f(X,X')\ne f^*(X,X')]^{\kappa'}. \]
Also suppose $f^*(x,x')=\sign(\Pr[Z=1|X=x,X'=x']-1/2)$.
Assume $\mathbb{C}'$ has VC-dimension $d'$ and disagreement coefficient $\theta'$, standard active learning requires $\Phi(\nu')=\tilde{\O}\left(\theta'\left(\frac{1}{\nu'}\right)^{2\kappa'-2}(d'\log(\theta')+\log(1/\delta))\log\left(\frac{1}{\nu'}\right)\right)$
samples to learn a comparison function of error $\nu'$ with probability $1-\delta$. So an algorithm $\mathcal{A}$ for adversarial noise on comparisons can be automatically transformed into an algorithm $\mathcal{A}'$ for TNC on comparisons with $\SC_{\lab}(\mathcal{A}')=\SC_{\lab}(\mathcal{A})$ and $\SC_{\comp}(\mathcal{A})=\Phi(\Tol_{\comp}(\mathcal{A}))$. So we only analyze adversarial noise for comparison in other parts of this paper.

%% file: proof_adgactsy.tex
\begin{proof}
	We only prove the theorem for $\kappa>1$, the case of $\kappa=1$ holds with a similar proof. An equivalent condition (see \cite{hanneke2014theory}) for Condition \ref{cond:labeltsy} under $\kappa>1$ is that there exists constant $\tilde{\mu}>0$ such that for all $t>0$ we have
\spad
\begin{align}
\Pr(|\eta(x)-1/2|<t)\leq \tilde{\mu}t^{1/{(\kappa-1)}}. \label{eqn:tsy2}
\end{align}
We use \eqref{eqn:tsy2} instead of Condition \ref{cond:labeltsy} through out the proof. 
    
    To bound the error in labeling by ADGAC, we first bound the number of incorrectly sorted pairs due to noise/bias of the comparison oracle.
    We call $(x_i,x_j)$ an \emph{inverse pair} if $h^*(x_i)=1,h^*(x_j)=-1,x_i\preccurlyeq x_j$  (the partial order is decided by randomly querying $Z(x_i,x_j)$ or $Z(x_j,x_i)$). Also, we call $(x_i,x_j)$ an \emph{anti-sort pair} if $h^*(x_i)=1,h^*(x_j)=-1, i<j $ (after sorting by Quicksort).
	Let $T$ be the set of all anti-sort pairs, $T'$ be the set of all inverse pairs in $S$, and $\tilde{T}'$ be the set of all inverse pairs in $\tilde{S}$. We first bound $|T|$ using $|T'|$. Let $s$ be the random bits supplied for Quicksort in its process, by Theorem 3 in \cite{ailon2007efficient} we have
	\[\E_{s}[|T|]=|T'|. \]
	%
	Notice that sampling a pair of $(X,X')$ is equivalent to sample a set $\tilde{S}$ of $n$ points and then uniformly pick two different points in it. Also, number of inverse pairs in $S$ is less than that in $\tilde{S}$. So we have
	\[\E_S\left[\E_{s}[|T|]\right]=\E_S[|T'|]\leq \E_{\tilde{S}}[|\tilde{T}'|]= n(n-1)\nu'\leq n^2\nu'.  \]
	By Markov inequality we have
	\begin{align}
	\label{eqn:markov}
	\Pr\left(|T|\geq \frac{2\nu'}{\delta}n^2\right)\leq \frac{\delta}{2}.
	\end{align}
	
	Suppose $|T|< \frac{2\nu'}{\delta}n^2$ (which holds with probability $>1-\delta/2$). We now proceed to bound the number of labeling errors made by ADGAC. First, notice that in Algorithm \ref{algo:adgac}, we divide all samples into groups of size $\alpha m=\varepsilon n/2$. 
	For every set $S_i$, let 
    \begin{align*}
    q(S_i)&=\frac{1}{|S_i|}\min\left\{\sum_{x\in S_i} I(h^*(x)=1),  \sum_{x\in S_i} I(h^*(x)=-1)\right\}\\
    &=\min\left\{\Pr_{X\sim S_i}(h^*(x)=-1),\Pr_{X\sim S_i}(h^*(x)=1) \right\}
    \end{align*}
    where $X\sim S_i$ denote the empirical distribution that $X$ is drawn uniformly at random from the finite collection of points in $S_i$. Let
	\[\beta=\frac{2}{\varepsilon}\sqrt{\frac{\nu'}{\delta}}\leq C\varepsilon^{\kappa-1} \]
	for some constant $C$. Suppose $\varepsilon$ is small enough such that $\beta\leq 1/2$. 
    Then we claim that there is at most 1 set $S_i$ such that $q(S_i)\geq \beta$.
	Otherwise, suppose two such sets exist; let them be $S_i$ and $S_j$. So there are at least $\alpha\beta m$ points $x\in S_i$ with $h^*(x)=-1$, and $\alpha\beta m$ points $x\in S_i$ with $h^*(x)=1$; the same holds for $S_j$. These -1s and 1s would indicate at least
	\[2\alpha^2\beta^2m^2=\frac{2\nu'}{\delta}n^2 \]
	anti-sort pairs, which violates our claim of $|T|$.
	
    Since ADGAC uses group binary search, we first analyze some properties of the majority label of the Bayes optimal classifier within each group/set. For each set $S_i$, let $\mu(S_i)=\sign(\sum_{x\in S_i} h^*(x_i))$ be the majority Bayes optimal label. We can show that $\mu(S_i)$ is monotonic: that is, for every $i<j$ we have $\mu(S_i)\leq \mu(S_j)$. To see this, suppose there exist two sets $S_i, S_j, i<j$ such that $\mu(S_i)=1$ and $\mu(S_j)=-1$. That would indicate at least $\alpha^2m^2/2>\alpha^2\beta^2m^2$ anti-sort pairs, which violates our assumption. So there must be a boundary $l$ such that $\mu(S_i)=-1$ for $i<l$, and $\mu(S_i)=1$ for $i\geq l$. We call $S_l$ to be the \emph{boundary set}. Now consider two cases: 
    \begin{itemize}
    \item Case 1: there exists a set $S_{l'}$ such that $q(S_{l'})\geq \beta$ (recall that from previous arguments, there is only one such set). If $l\ne l'$, the sets $S_l$ and $S_{l'}$ generates at least $2(\alpha m/2)(\alpha \beta m)\geq 2\alpha^2\beta^2m^2$ anti-sort pairs, which violates our assumption for $|T|$. So $l=l'$.
    \item Case 2: for all sets $S_i$, we have $q(S_i) < \beta$. 
    \end{itemize}
    In both cases, we have $q(S_i)<\beta$ for all $i\ne l$.


Now we prove that the majority vote of the noisy labels agrees with the majority vote of the Bayes optimal classifier $\mu(S_i)$ for each set $S_i$ that we visit, and hence we will find the boundary set $S_l$.	
Suppose $q(S_i)<\beta$. Take
	\[t=\left(\frac{\varepsilon}{16\tilde{\mu}}\right)^{\kappa-1}. \]
	For small enough $\varepsilon$, we have $t\leq 1/2$, and $\Pr(x:|\eta(x)-1/2|\leq t)\leq \varepsilon/16$.
	Let $U=\{x_i\in S: |\eta(x_i)-1/2|\leq t \}$. By relative form of Chernoff bound we have
	\[\Pr\left(|U|>3\log(4/\delta)+n\varepsilon/8 \right)\leq \exp\left(-\frac{3\log(4/\delta)+\varepsilon n/16}{3}\right)\leq \frac{\delta}{4}. \]
	Suppose $|U|/n\leq \varepsilon/8$, so at most 1/4 of each $S_i$ is in $U$.
	
	For each set $S_i$, let $\bar{S}_i=\{x\in S_i: h^*(x)\ne \mu(S_i)\}$ and $S'_i=\{x\in S_i: |\eta(x)-1/2|\leq t\}$. So for each set such that $q(S_i)\leq \beta$, we have
	\begin{align*}
	\Pr(Y\ne \mu(S_i)|X\sim S_i)\leq& \Pr(Y\ne \mu(S_i)|X\in S_i')\Pr(X\in S_i'|X\sim S_i)+\\
	&\Pr(Y\ne \mu(S_i)|X\in \bar{S}_i)\Pr(X\in \bar{S}_i|X\sim S_i)+\\
	&\Pr(Y\ne \mu(S_i)|X\not\in S'_i,X\not\in \bar{S}_i)\Pr(X\not\in S'_i,X\not\in \bar{S}_i|X\sim S_i)\\
	\leq& \left(\frac{1}{2}+t\right)\cdot \frac{1}{4}+1\cdot \beta+\left(\frac{1}{2}-t\right)\cdot \left(\frac{3}{4}-\beta\right) \\
	=& \frac{1}{2}-\frac{1}{2}t+\left(\frac{1}{2}+t\right)\beta.
	\end{align*}
	Pick $\nu'$ small enough such that $\beta\leq \frac{1}{4}t$:
	\[\frac{2}{\varepsilon}\sqrt{\frac{\nu'}{\delta}}\leq \frac{1}{4}\left(\frac{\varepsilon}{16\tilde{\mu}}\right)^{\kappa-1}. \]
	This yields
	\[\nu'\leq \frac{ \varepsilon^{2\kappa}\delta}{32(16\tilde{\mu})^{2\kappa-2}}. \]
	Note that this also guarantees $\beta\leq 1/2$ above, since $t\leq \frac{1}{2}$. Now we have
	\[\Pr\left[Y\ne \mu(S_i)|X\sim S_i\right]\leq \frac{1}{2}-\frac{1}{2}t+\frac{1}{4}t\left(t+1/2\right)\leq \frac{1}{2}-\frac{1}{4}t. \]
	In the algorithm, suppose we pick $X_1,X_2,...,X_k\in S_i$ as the points for which to query the label and the labels are $Y_1,...,Y_k$. By Hoeffding's inequality, we have
	\[\Pr\left[\sign\left(\sum_{j=1}^k Y_j\right)=\mu(S_i)\right]=\Pr\left[\frac{1}{k}\sum_{j=1}^n I(Y_j=\mu(S_i))>\frac{1}{2}\right]\leq \exp\left(-\frac{1}{8}k t^2\right). \]
	The choice of $k$ yields that the majority vote of the noisy labels agrees with the majority vote $\mu(S_i)$ of the Bayes optimal classifier for each $S_i$ with $q(S_i) \leq \beta$ we visit, with probability $\frac{\delta}{8\log(2/\varepsilon)}$.
	
    Suppose the binary search output set $S_t$ (i.e., the value of $t$ at step \ref{step:havet}). Now we analyze the errors we made in the final output. We consider the two cases:
    \begin{itemize}
    \item Case 1: If $q(S_l)\geq \beta$, then with probability $1-\delta$, we have $t\in \{l-1,l,l+1\}$ since we might behave arbitrarily in set $S_l$. In this case, we have $q(S_l)|S_l|\geq \alpha \beta m$, and so $|\{x:x\in S_{t'},t'<l,h^*(x)=1\}\leq \alpha \beta m$, because otherwise we have $\alpha^2\beta^2m^2$ anti-sort pairs, which violates our assumption on $|T|$. Similarly, $|\{x:x\in S_{t'},t'>l,h^*(x)=-1\}|\leq \alpha \beta m$. Counting also the possible errors made on $S_l$, the total number of errors is
    \[|\{\hat{y}_i:\hat{y}_i\ne h^*(x_i) \}|\leq \alpha m+2\alpha \beta m \leq \frac{\varepsilon n}{2}\left(1+\frac{1}{2}t\right)\leq \varepsilon n. \]
    \item Case 2: If $q(S_i)<\beta$ for all $i$, then we have $t\in \{l-1,l\}$. Now note that we have $|\{x\in S_{l-1}: h^*(x)=-1\}|\geq \alpha m/2$, and so $|\{x\in S_{t'}: t'<l-1, h^*(x)=1\}|\leq \alpha \beta m$ since otherwise at least $\alpha^2\beta m/2$ anti-sort pairs are present. So $|\{x\in S_{t'}: t'\leq l-1, h^*(x)=1\}|\leq 2\alpha \beta m$ considering $q(S_{l-1})<\beta$. Similarly, $|\{x\in S_{t'}: t'\geq l, h^*(x)=-1\}|\leq 2\alpha \beta m$. So the total number of errors is
     \[|\{\hat{y}_i:\hat{y}_i\ne h^*(x_i) \}|\leq 4\alpha \beta m \leq \varepsilon n. \]
    \end{itemize}
	So we have at most $\varepsilon n$ error under both cases.
	Now we examine the total query complexity: It takes $k=\tilde{\O}\left(\log(1/\delta)\left(\frac{1}{\varepsilon}\right)^{2\kappa-2}\right)$ queries for each set $S_t$, and we do this for $\O(\log(1/\alpha))=\O\left(\log\left(\frac{2m}{\varepsilon n}\right)\right)$ times. So the total query complexity is $$\tilde{\O}\left(\log\left(\frac{2m}{\varepsilon n}\right)\log(1/\delta)\left(\frac{1}{\varepsilon}\right)^{2\kappa-2}\right).$$
\end{proof}

%% file: proof_adgacagn.tex
\begin{proof}
	The first part of proof is exactly the same as that of Theorem \ref{thm:adgactsy}. We now bound $\Pr[Y\ne \mu(S_i)|X\sim S_i]$. Suppose $q(S_i)<\beta$. Let $V=\{x: \Pr[Y\ne h^*(X)|X=x]>1/4 \}$ and $U=\{x_i: \Pr[Y\ne h^*(X)|X=x_i]>1/4 \}$. We have $P(V)\leq 4\nu$. By a relative Chernoff bound, if $\nu\leq C_1\varepsilon$ for a small enough constant $C_1$ we have
	\[\Pr[|U|\leq 8n\nu+3\log(4/\delta)]\leq \exp\left(-\frac{3\log (4/\delta)+4\nu n}{3}\right)\leq \delta/4. \]
	So if $\nu\leq \frac{1}{64}\varepsilon$ we have $|U|/n\leq \varepsilon/8$ with probability $\delta/4$. In this case, at most 1/4 of each $S_i$ is in $U$.
	
	For each set $S_i$, let $\bar{S}_i=\{x\in S_i: h^*(x)\ne \mu(S_i)\}$ and $\tilde{S}_i=\{x\in S_i, x\in U\}$. So for each set such that $q(S_i)\leq \beta$, we have
	\begin{align*}
	\Pr(Y\ne \mu(S_i)|X\sim S_i)\leq& \Pr(Y\ne \mu(S_i)|X\in \tilde{S}_i)\Pr(X\in \tilde{S}_i|X\sim S_i)+\\
	&\Pr(Y\ne \mu(S_i)|X\in \bar{S}_i)\Pr(X\in \bar{S}_i|X\sim S_i)+\\
	&\Pr(Y\ne \mu(S_i)|X\not\in \tilde{S}_i,X\not\in \bar{S}_i)\Pr(X\not\in \tilde{S}_i,X\not\in \bar{S}_i|X\sim S_i)\\
	\leq& 1\cdot \frac{1}{4}+1\cdot \beta+\left(\frac{3}{4}-\beta\right)\frac{1}{4}\\
	=& \frac{7}{16}+\frac{3}{4}\beta.
	\end{align*}
	So there exists constant $C_2$ such that if $\nu'\leq C_2\varepsilon^2\delta$ we have $\beta\leq \frac{1}{24}$, $\Pr(Y\ne \mu(S_i)|X\sim S_i)]\leq \frac{1}{2}-\frac{1}{32}$. 
    Thus by Hoeffding's inequality, the choice of $k$ yields that we recover $\mu(S_i)$ for each $i$ we visit with probability $\frac{\delta}{8\log(2/\varepsilon)}$.
	
	By similar analysis as the proof of Theorem \ref{thm:adgactsy}, we can show that the number of errors (i.e., $|\{\hat{y}_i: \hat{y}_i\ne h^*(x_i)\}|$) is at most $\varepsilon n$.
	
	
	Now examine the total query complexity: It takes $k=\O\left(\log(\log(1/\varepsilon)/\delta)\right)$ queries for each set $S_t$, and we do this for $\O(\log(1/\alpha))=\O(\log\left(\frac{2m}{\varepsilon n}\right))$ times. So the total query complexity is $$\O\left(\log\left(\frac{2m}{\varepsilon n}\right)\log\left(\frac{\log(1/\varepsilon)}{\delta}\right)\right).$$
\end{proof}	 

%% file: proof_a2adgac.tex
We use the following lemma adapted from \cite{hanneke2014theory}:
\begin{lemma}[\cite{hanneke2014theory}, Lemma 3.1]
	\label{lemma:full}
	Suppose that $\mathcal{D}=\{x_1,x_2,...,x_n\}$ is i.i.d. sampled from $\P_\mathcal{X}$, and $h^*\in \bC$. There is a universal constant $c_0\in (1,\infty)$ such that for any $\gamma\in (0,1)$, and any $n\in \mathbb{N}$, letting
	\begin{align*}
	U(n,\gamma)=c_0\frac{d\log(n/d)+\log(1/\gamma)}{n},
	\end{align*}
	with probability at least $1-\gamma$, $\forall h\in \mathbb{C}$, the following inequalities hold:
	\begin{align*}
	\Pr_{X\sim \P_\mathcal{X}}[h(X)\ne h^*(X)]&\leq \max\{2\Pr_{X\sim \mathcal{D}}[h(X)\ne h^*(X)],U(n,\gamma) \},\\
	\Pr_{X\sim \mathcal{D}}[h(X)\ne h^*(X)]&\leq \max\{2\Pr_{X\sim \P_\mathcal{X}}[h(X)\ne h^*(X)],U(n,\gamma) \}.
	\end{align*}
    Here $X\sim \mathcal{D}$ means $X$ is uniformly sampled from finite set $\mathcal{D}$.
\end{lemma}

\begin{algorithm}[htb]
	\caption{A$^2$-ADGAC}
	\begin{algorithmic}[1]		
		\Require{$n_i,\mathbb{C},\varepsilon,\delta$, comparison oracle $f$.}
		\State Let $V\leftarrow \mathbb{C}$.
		\For{$i=1,2,...,\lceil\log(1/\varepsilon)\rceil$}
		\State Sample dataset $\tilde{S}$ of size $n_i$.
		\State Let $S\leftarrow \{x\in \tilde{S}:x\in \DIS(V) \}$.
		\State Run ADGAC (Subroutine \ref{algo:adgac}) with $S, \tilde{S},\varepsilon_i=2^{-(i+2)}, k_i$ and labeled dataset $W$.
		\State $V=V\setminus \{h:|W|\err_W(h)\geq n_i\varepsilon_i \}$.\label{step:filter}
		\EndFor
        \Ensure{Any Classifier $\hat{h}\in V$.}
	\end{algorithmic}
	\label{algo:a2adgac}
\end{algorithm}
\begin{proof}[Proof of Theorem \ref{thm:a2adgactsy}]
	For a labeled dataset $W=\{(x_i,\hat{y_i})\}_{i=1}^n$, let $\err_W(h)=\frac{1}{n}\sum_{i=1}^n I(h(x_i)\ne \hat{y_i})$ be the empirical risk of $h$ on $W$ for any $h\in \bC$ (remind that $\hat{y}_i$ are predictions of ADGAC). For a clearer explanation, we formalize the A$^2$-ADGAC algorithm in Algorithm \ref{algo:a2adgac}.
	We use induction to prove that after iteration $i$ we have $\Pr[h(X)\ne h^*(X)]\leq 4\varepsilon_i$ for all $h\in V$ after step \ref{step:filter} in Algorithm \ref{algo:a2adgac}. This proposition holds for $i=0$. Suppose it holds for $i-1$. By Theorem \ref{thm:adgactsy} and a union bound, with probability $1-\delta/4$, for every iteration $i$ we have at most $n_i\varepsilon_i$ errors with respect to $h^*$ after running ADGAC, i.e., $|W|\err_W(h^*)\leq n_i\varepsilon_i$. So $h^*$ will not be eliminated from $V$ in any iteration with probability $1-\delta/4$. 
    On the other hand, notice that by Step \ref{step:filter} in Algorithm \ref{algo:a2adgac} all functions $h\in V$ satisfies $|W|\err_W(h)\leq n_i\varepsilon_i$, so by triangle inequality we have (notice that $W$ is just the set $S$ with labels)
    \begin{align*}
    |S|\Pr_{X\sim S}[h(X)\ne h^*(X)] &=\left|\{x\in S: h(x)\ne h^*(x) \}\right|\\
    &\leq \left|\{(x,\hat{y})\in W: h(x)\ne \hat{y} \}\right|+\left|\{(x,\hat{y})\in W: h^*(x)\ne \hat{y} \}\right|\\
    &\leq 2\varepsilon_in_i.
    \end{align*}
    Also note that functions in $V$ agrees on $\tilde{S}\setminus S$; so $|\tilde{S}|\Pr_{X\sim \tilde{S}}[h(X)\ne h^*(X)]\leq 2\varepsilon_i n_i$, and since $|\tilde{S}|=n_i$ we have $\Pr_{X\sim \tilde{S}}[h(X)\ne h^*(X)]\leq 2\varepsilon_i$. Now using Lemma \ref{lemma:full} with $n=n_i$, we have $\Pr_{X\sim \P_{\mathcal{X}}}[h(x)\ne h^*(x)]\leq 4\varepsilon_i$ for every $h\in V$ by choosing $n_i$ such that $U\left(n_i,\frac{\delta}{4\log(1/\varepsilon)}\right)\leq \varepsilon_i$.
	So at the end of the algorithm it outputs a classifier with $\Pr[\hat{h}\ne h^*]\leq \varepsilon$.
	
	Now we examine the number of queries. By definition of disagreement coefficient, at round $i$ we have $\DIS(V)\leq \theta \varepsilon_i$; thus using a relative Chernoff bound we know that with probability $1-\delta/4$ we have
	\[m_i\coloneqq|S|\leq \log(12/\delta)+2n_i\theta \varepsilon_i=\dd \O\left(\theta \left(\left(d\log(1/\varepsilon)\right)+\left(\frac{1}{\varepsilon_i}\right)^{2\kappa-2}\log(1/\delta)\right)\right). \]
	It takes $O(m_i\log m_i)$ comparisons in expectation to rank the set, and there are $\log(1/\varepsilon)$ iterations. So the total comparison complexity is
	\[\E[\SC_{\comp}]\hspace{-0.1cm}=\hspace{-0.1cm}\tilde{\O}\left(\theta\log\left(\frac{1}{\varepsilon}\right)\left(\log d\theta+ (\kappa-1)\log\left(\frac{1}{\varepsilon}\right)\right)\hspace{-0.1cm}\left(\left(d\log\left(\frac{1}{\varepsilon}\right)\right)+\left(\frac{1}{\varepsilon}\right)^{2\kappa-2}\log(1/\delta)\right)\hspace{-0.1cm}\right).\]
	This obtained the stated comparison complexity.
	The label complexity follows by multiplying the label complexity of ADGAC by $\log(1/\varepsilon)$. Note that in every step we have $\frac{m_i}{\varepsilon_i n_i}=O\left(\min\left\{\theta,\frac{1}{\varepsilon}\right\}\right)$.
\end{proof}
\begin{proof}[Proof of Theorem \ref{thm:a2adgacagn}]
	With the same proof, A$^2$-ADGAC outputs a classifier with $\Pr[\hat{h}\ne h^*]\leq \varepsilon$ upon finishing. We know examine the number of queries. By definition of disagreement coefficient, at round $i$ we have $\DIS(V)\leq \theta \varepsilon_i$; thus using a Chernoff bound we know that with probability $1-\delta/4$ we have
	\[m_i\coloneqq|S|\leq \log(12/\delta)+2n_i\theta \varepsilon_i=\dd \O\left(\theta d\log\left(\frac{1}{\varepsilon_i}\right)\log\left(\frac{1}{\delta}\right)\right). \]
	It takes $O(m_i\log m_i)$ comparisons in expectation to rank the set, and there are $\log(1/\varepsilon)$ iterations. So the total comparison complexity is
	\[\E[\SC_{\comp}]=\tilde{\O}\left(\theta d \log(\theta d) \log\left(\frac{1}{\varepsilon_i}\right)\log\left(\frac{1}{\delta}\right)\right).\]
	This obtained the stated comparison complexity.
	The label complexity follows by multiplying the label complexity of ADGAC by $\log(1/\varepsilon)$. Note that in every step we have $\frac{m_i}{\varepsilon_i n_i}=O\left(\min\left\{\theta,\frac{1}{\varepsilon}\right\}\right)$.
\end{proof}

%% file: proof_localadgac.tex
\begin{algorithm}[htb]
	\caption{Margin-ADGAC: Efficiently learning halfspaces with comparison}
	\begin{algorithmic}[1]		
%
		\Require{$\varepsilon,\delta$, target errors $\varepsilon_k$, sample sizes $n_k$, sequences $r_k,b_k,\tau_k$, precision value $\kappa$.}
		\State Draw $n_1$ unlabeled samples in $S$ and run ADGAC with $\left(S,n_1,\varepsilon_0,\frac{\delta}{8\log(1/\varepsilon)}, k_1\left(\varepsilon_0,\frac{\delta}{8\log(1/\varepsilon)}\right)\right)$, and obtain a labeled dataset $W$.
		\For{$k=1,2,...,s=\lceil\log(4/\varepsilon)\rceil$}
		\State Find $v_k\in B(w_{k-1},r_k)$ that approximately minimize training hinge loss over $W$, with $\|v_k\|_2\leq 1$:
		\[l_{\tau_k}(v_k,W)\leq \min_{w\in B(w_{k-1},r_k)\cap B(0,1)} l_{\tau_k}(w,W)+\kappa/8.\]
		\State $w_k\leftarrow \frac{v_k}{\|v_k\|_2}$.
		\State Sample another dataset $\tilde{S}$ of $n_k$ unlabeled samples.
		\State $S=\{x\in S:|w_k\cdot x|\leq b_k\}$.
		\State Run ADGAC with $\left(S,n_k,\varepsilon_k,\frac{\delta}{8\log(1/\varepsilon)},k^{(1)}\left(\varepsilon_k,\frac{\delta}{8\log(1/\varepsilon)}\right)\right)$ and obtain labeled dataset $W$.
		\EndFor
		\Ensure{Return $w_s$.}
	\end{algorithmic}
	\label{algo:localadgac}
\end{algorithm}

We first prove Theorem \ref{thm:localadgactsy}, and Theorem \ref{thm:localadgacagn} follows exactly the same proof with $\kappa=1$ and using Theorem \ref{thm:adgacagn}. For clearer explanation, we re-illustrate Margin-ADGAC in a form similar to that in \cite{awasthi2017power} in Algorithm \ref{algo:localadgac}.
The proof mostly follows that of \cite{awasthi2017power}. We give a refined sample complexity via Rademacher complexity following the ideas in \cite{yan2017revisiting}, and also change the proof according to the properties of ADGAC (note that we are not using independent samples by replace the sampling step with ADGAC).

To simplify notations, let $\err(w)$ be  $\err(h_w(x))=\err(\sign(w\cdot x))$. Define $\Delta_{D}(w,w')=\Pr_{X\sim D}[\sign(w\cdot X)\ne \sign(w'\cdot X)]$.
Also, let $\theta(w_1,w_2)$ be the angle between two vectors $w_1, w_2$. 
Let $D_{w,\gamma}=\{x:|w\cdot x|\leq \gamma\}$. 

The key step is to prove the following theorem:
\begin{theorem}\label{thm:keylocaladgac}
	For $k\leq \log(1/\varepsilon)$, if $\Delta_{\P_\mathcal{X}}(w_{k-1},w^*)\leq M^{-(k-1)}$, with probability $1-\frac{\delta}{k+k^2}$, after round $k$ of Margin-ADGAC we have $\Delta_{D_{w_{k-1},b_{k-1}}}(w_k,w^*)\leq \kappa$.
\end{theorem}
To prove the theorem, we first list useful properties of isotropic log-concave distributions and fix the parameters we use for the algorithm. We use exactly the same parameters for $r_i,\tau_i,b_i,z_i$ as in \cite{awasthi2017power}, and we restate them here for completeness.

\begin{lemma}[\cite{awasthi2017power,balcan2013active,lovasz2007geometry}]\label{lem:ilc}
	Suppose $X\sim \mathcal{P}_\mathcal{X}$ is a isotropic log-concave distribution in $\mathbb{R}^d$ with probability density function $f$. Then
	\begin{enumerate}[(a)]
		\item There is an absolute constant $c_1$ such that, if $d=1$, $f(x)>c_1$ for all $x\in [-1/9,1/9]$.
		\item There is an absolute constant $c_2$ such that for any two unit vectors $u$ and $v$ in $R^d$ we have $c_2\theta(u,v)\leq \Delta_{\P_\mathcal{X}}(u,v)$.
		\item There exists constant $c_3$ such that for any unit vector $w$ and $\gamma>0$, $\Pr[|w\cdot X|\leq \gamma]\leq c_3\gamma$.\label{item:boundprob}
		\item There is a constant $c_4$ such that for any unit vector $u$, all $0<\gamma<1$, for all $a$ such that $\|u-a\|_2\leq \gamma$ and $\|a\|_2\leq 1$, $\E_{X\sim D_{u,\gamma}}[(a\cdot X)^2]\leq c_4(r^2+\gamma^2)$. \label{item:boundsquare}
		\item For any $c_5>0$, there is a constant $c_6>0$ such that the following holds: let $u$ and $v$ be two unit vectors in $\mathbb{R}^d$, and assume that $\theta(u,v)=\alpha<\pi/2$. Then $\Pr_{X \sim\P_{\mathcal{X}}}[\sign(u\cdot X)\ne \sign(v \cdot X) \text{ and }|v\cdot X|\leq c_6\alpha]\leq c_5\alpha$.
	\end{enumerate}
\end{lemma}

Now we give the settings of parameters. Let $M=\max\{\frac{2}{c_2\pi},2 \}$. Let $c_1'$ be the value of $c_6$ in Lemma \ref{lem:ilc} corresponding to the case where $c_5$ is $\frac{c_2}{4M}$; let $b_k=c_1'M^{-k}$. Let $r_k=\min\{M^{-(k-1)}/c_2, \pi/2 \}$ and $\kappa=\frac{1}{4c_1'M}$. Let $\tau_k=\frac{c_1\min\{b_{k-1},1/9 \}\kappa}{6}$, and $z_k^2=r_k^2+b_{k-1}^2$. Let $\varepsilon_k=\frac{c_3\tau_k^2b_k\kappa^2}{256 c_4 z_k^2}$, and $n_k=O\left(\frac{1}{b_k}d\log^3\left(\frac{dk}{1/\delta}\right)\right)$. Also let $m_k=2c_3b_kn_k+\log(12k/\delta)$.

Then we prove the following lemma:
\begin{lemma}\label{lem:boundl}
	Suppose $|W|\geq m_k$. Let $c(W)$ be the set with truthful labels w.r.t. $w^*$, i.e., $c(W)=\{(x,\sign(w^*\cdot x)):x\in W\}$. For any $w\in B(w_{k-1},r_k)$, with probability $1-\frac{\delta}{3(k+k^2)}$ we have
	\[|l(w,W)- l(w,c(W))|\leq \kappa/8.\]
\end{lemma}
\begin{proof}
	Let $N=\{(x,\hat{y})\in W:\hat{y}\ne \sign(w^*\cdot x) \}$ be the set where ADGAC has $x$'s label different than $\sign(w^*\cdot x)$ (remind that $\hat{y}$ is the prediction of ADGAC). We have
	\begin{align*}
	l(w,W)&=\frac{1}{|W|}\sum_{(x,\hat{y})\in W} l_{\tau_k}(w,x,\hat{y})\\
	&=\frac{1}{|W|}\left(\sum_{(x,y)\not\in N}l_{\tau_k}(w,x,\sign(w^*\cdot x))+ \sum_{(x,y)\in N}l_{\tau_k}(w,x,-\sign(w^*\cdot x)) \right).\\
	\end{align*}
	So
	\begin{align}
	|l(w,W)- l(w,c(W))|&\leq \frac{1}{\tau_k|W|}\sum_{x\in N}2(w\cdot x)\nonumber\\
	&\leq \frac{1}{\tau_k|W|}\sum_{x\in W}I(x\in N) 2(w\cdot x).\label{eqn:relationhinge}
	\end{align}
	We use the following lemma from \cite{awasthi2017power}:
	\begin{lemma}[Lemma D.4, \cite{awasthi2017power}]
		For an absolute constant $c$, with probability $1-\frac{\delta}{6(k+k^2)}$,
		\[\max_{x\in W} \|x\|_2\leq c\sqrt{d}\log\left(\frac{|W|k}{\delta}\right). \]
	\end{lemma}
	Note that
	\[|w\cdot x|\leq |w_{k-1}\cdot x|+|(w-w_{k-1})\cdot x|\leq b_k+r_k\|x\|_2. \]
	So with probability $1-\frac{\delta}{6(k+k^2)}$, an event $E_\delta$ happens such that
	\[\frac{|w\cdot x|}{\tau_k}\leq c'\sqrt{d}\log\left(\frac{|W|k}{\delta}\right) \]
	for all $x\in W$, for some constant $c'$.
	
	Notice that $\frac{|N|}{|W|}\leq \frac{\varepsilon_k n_k}{m_k}$. Let $N'$ be a $\frac{\varepsilon_k n_k}{m_k}$ fraction of $W$ with the largest values of $|w\cdot x|$. Let $\varphi(W)=\sum_{x\in N'} |w\cdot x|$. So by \eqref{eqn:relationhinge} we have $|l(w,W)-l(w,c(W))|\leq \frac{2}{\tau_k|W|}\varphi(W)$. Now we have
	\begin{align*}
	\E[\varphi(W)]&=\E\left[\sum_{x\in W} \delta(x\in N')|w\cdot x|\right]\\
	&\leq \sqrt{\frac{|N'|}{|W|}}\E\left[\sqrt{\sum_{x\in W} (w\cdot x)^2}\right]\\
	&\leq \sqrt{\frac{\varepsilon_k n_k}{m_k}} \sqrt{\E\left[\sum_{x\in W} (w\cdot x)^2\right]}\\
	&\leq \sqrt{\frac{\varepsilon_k n_k}{m_k}} \sqrt{c_4}z_k|W|\leq \kappa \tau_k |W|/16.
	\end{align*}
	The first inequality is by Cauchy-Schwartz inequality; the second is by Jensen's inequality; the third inequality is by property \ref{item:boundsquare} in Lemma \ref{lem:ilc}; the last inequality is by the value of $\varepsilon_i$. If we condition $\P_\mathcal{X}$ on $E_\delta$, the above expectation will be smaller since we bound $|w\cdot x|$ from above. Now by Mcdiarmid's inequality, $\frac{1}{|W|}\varphi(W)$ deviates by at most $\frac{c'\sqrt{d}\log\left(\frac{|W|k}{\delta}\right)}{|W|}$ when we change a single value of $w\cdot x$ for some $x\in W$. So by McDiarmid's inequality, using $|W|\geq m_k=\Omega(d\log^2(d/\delta))$, with probability $1-\frac{\delta}{3(k+k^2)}$ we have
	\[|l(w,W)-l(w,c(W))|\leq \E[\varphi(W)|E_\delta]+\kappa/16\leq \kappa/8.  \]
\end{proof}

The other lemma is about bounding the difference between $l(w,c(W))$ and $E_W[l(w,c(W))]$. We improve the results in \cite{awasthi2017power} using Rademacher complexity as below.

\begin{lemma}\label{lem:boundE}
	With probability $1-\frac{\delta}{6(k+k^2)}$ we have
	\[|\E_W[l(w,x,\sign(w^*\cdot x))]-l(w,W)|\leq \kappa/16. \]
\end{lemma}
\begin{proof}
	Note that every $x\in W$ is sampled independently from $D_{w_k,b_{k-1}}$. Following the same proof as in Lemma \ref{lem:boundl}, an Event $E_\delta$ happens with probability $1-\frac{\delta}{6(k+k^2)}$ that 	
	\[\left|\frac{w\cdot x}{\tau_k}\right|\leq c'\sqrt{d}\log\left(\frac{|W|k}{\delta}\right) \]
	for all $x\in W$, for some constant $c'$. This means $l_{\tau_k}(w,x,\sign(w^*\cdot x))$ are also bounded in the same range under $E_\delta$.
	
	Define the function class $\mathcal{F}=\{x\rightarrow l_{\tau_k}(w,x,\sign(w^*\cdot x)), \|w-w_k\|\leq r_k \}$. On event $E_\delta$, all functions in $\mathcal{F}$ are bounded. Now we bound the Rademacher complexity $R_n(\mathcal{F})$. Actually, define $\mathcal{F}'=\{x\rightarrow \frac{1}{\tau_k}w\cdot x\cdot \sign(w^*\cdot x), \|w-w_k\|\leq r_k \}$, we have $R_n(\mathcal{F})\leq R_n(\mathcal{F}')$ by contraction inequality of Rademacher complexity (since hinge loss is 1-Lipschitz). So
	\begin{align}
	R_n(\mathcal{F})&\leq R_n(\mathcal{F}')\nonumber\\
	&=\frac{1}{\tau_kn}E_{x_1,...,x_n\sim D_{w_k,b_{k-1}}}E_{\sigma_1,...,\sigma_n} \sup_{w:\|w-w_k\|\leq r_{k}} \sum_{i=1}^n \sigma_i w\cdot x_i \cdot\sign(w^*\cdot x_i)\nonumber\\
	&= \frac{1}{\tau_kn}E_{x_1,...,x_n\sim D_{w_k,b_{k-1}}}E_{\sigma_1,...,\sigma_n}\sup_{w:\|w-w_k\|\leq r_{k}}\sum_{i=1}^n \sigma_i(w\cdot x_i)\label{eqn1}\\
	&= \frac{1}{\tau_kn}E_{x_1,...,x_n\sim D_{w_k,b_{k-1}}}E_{\sigma_1,...,\sigma_n}\sum_{i=1}^n\sigma_i(w_k\cdot x_i)+\nonumber\\
	& \frac{1}{\tau_kn}E_{x_1,...,x_n\sim D_{w_k,b_{k-1}}}E_{\sigma_1,...,\sigma_n}\sup_{w:\|w-w_k\|\leq r_{k}}\sum_{i=1}^n\sigma_i(w-w_k)\cdot x_i\nonumber\\
	&= \frac{1}{\tau_kn}E_{x_1,...,x_n\sim D_{w_k,b_{k-1}}}E_{\sigma_1,...,\sigma_n}\sup_{w:\|w-w_k\|\leq r_{k}}\sum_{i=1}^n\sigma_i(w-w_k)\cdot x_i\nonumber\\
    &= \frac{1}{\tau_kn}E_{x_1,...,x_n\sim D_{w_k,b_{k-1}}}E_{\sigma_1,...,\sigma_n}\sup_{w:\|w-w_k\|\leq r_{k}}(w-w_k)\sum_{i=1}^n\sigma_i\cdot x_i\nonumber\\
	&\leq \frac{1}{\tau_kn}E_{x_1,...,x_n\sim D_{w_k,b_{k-1}}}E_{\sigma_1,...,\sigma_n}\sup_{w:\|w-w_k\|\leq r_{k}}\|w-w_k\|_2\left\|\sum_{i=1}^n\sigma_ix_i\right\|_2\nonumber\\
	&\leq \frac{2r_k}{\tau_k n}\sqrt{E_{x_1,...,x_n\sim D_{w_k,b_{k-1}}}E_{\sigma_1,...,\sigma_n}\left\|\sum_{i=1}^n\sigma_ix_i\right\|^2_2}\label{eqn2}\\
    &\leq \frac{2r_k}{\tau_k n}\sqrt{E_{x_1,...,x_n\sim D_{w_k,b_{k-1}}}E_{\sigma_1,...,\sigma_n}\left[\sum_{i=1}^n \|x_i\|_2^2+\sum_{i,j} \sigma_i\sigma_j x_i\cdot x_j\right]}\\
	&\leq \O\left(\frac{1}{n}\cdot \sqrt{nd\log^2\left(\frac{nk}{\delta}\right)}\right)\label{eqn3}\\
	&= \O\left(\sqrt{\frac{d\log^2\left(\frac{nk}{\delta}\right)}{n}}\right).\nonumber
	\end{align}
	\eqref{eqn1} is by the property that $\sigma_i \cdot \sign(w^*\cdot x_i)$ has the same distribution as $\sigma_i$, and thus we can substitute $\sigma_i \cdot \sign(w^*\cdot x_i)$ with a single variable; \eqref{eqn2} is by Jensen's inequality, and \eqref{eqn3} is by the boundary condition on $\|x\|_2$. So by Rademacher's inequality we have
	\begin{align*}
	|\E_W[l(w,x,\sign(w^*\cdot x))]-l(w,W)|&\leq R_{|W|}(\mathcal{F})+\sqrt{\frac{\log(1/\delta)}{|W|}}C\sqrt{d}\log\left(\frac{|W|k}{\delta}\right) \\
	&\leq \O\left(\sqrt{\frac{d\log^2\left(\frac{|W|k}{\delta}\right)}{|W|}}\right)+\sqrt{\frac{\log(k/\delta)}{|W|}}C\sqrt{d}\log\left(\frac{|W|k}{\delta}\right)\\
	&=\O\left(\sqrt{\frac{d\log(k/\delta)}{|W|}}\log\left(\frac{|W|k}{\delta}\right)\right).
	\end{align*}
	The choice of $|W|=m_i=\Omega\left(d\log^3\left(\frac{dk}{1/\delta}\right)\right)$ makes the above quantity less than $\kappa/16$.
\end{proof}
Now we are ready to prove Theorem \ref{thm:keylocaladgac}.
\begin{proof}[Proof of Theorem \ref{thm:keylocaladgac}.]
	With a probability of $1-\frac{\delta}{k+k^2}$, suppose the conditions in Lemma \ref{lem:boundl} and \ref{lem:boundE} holds for $w=v_k$ and $w=w^*$. We have
	\begin{align*}
	 \;& \Delta_{D_{w_{k-1},b_{k-1}}}(w_k,w^*)\\
    =\;&\Delta_{D_{w_{k-1},b_{k-1}}}(v_k,w^*)\\
	\leq\;& \E_{x\in D_{w_{k-1},b_{k-1}}}[l(v_k,x,\sign(w^*\cdot x))] \;\; &\text{(Since hinge loss upper bounds 0-1 loss)}\\
	\leq\;& l(v_k,c(W))+\kappa/16 & \text{(Using Lemma \ref{lem:boundE})}\\
	\leq\;& l(v_k,W)+\kappa/8 & \text{(Using Lemma \ref{lem:boundl})}\\
	\leq\;& l(w^*,W)+\kappa/4 & \text{(By the process of selecting $v_k$)}\\
	\leq\;& l(w^*,c(W))+\kappa/4+\kappa/16& \text{(Using Lemma \ref{lem:boundl})}\\
	\leq\;& L(w^*)+\kappa/4+\kappa/8& \text{(Using Lemma \ref{lem:boundE})}\\
	\leq\;& \kappa. & \text{(Using Lemma 3.7 in \cite{awasthi2017power})}
	\end{align*}	
\end{proof}
Now we can prove Theorem \ref{thm:localadgactsy}.
\begin{proof}[Proof of Theorem \ref{thm:localadgactsy}.]
	By relative Chernoff bound and property \ref{item:boundprob} in Lemma \ref{lem:ilc}, with probability $1-\frac{\delta}{6(k+k^2)}$ we have $|W|\geq m_k=2c_3b_kn_k+\log(12k/\delta)$ in every iteration. Then the correctness of Margin-ADGAC follows the same way as in \cite{awasthi2017power}. Now examine the number of queries: In each step we need to compare $m_i$ instances, as well as fitting the minimum requirement of ADGAC. So the comparison complexity is
  \[\E[\SC_{\comp}]=\tilde{\O}\left(\log^2(1/\varepsilon)\left(d\log^4(d/\delta)+\left(\frac{1}{\varepsilon}\right)^{2\kappa-2}\log(1/\delta)\right)\right).\]
  The label complexity is again obtained by multiplying the label complexity in each iteration by $\log(1/\varepsilon)$. Note that $\frac{\varepsilon_k n_k}{m_k}$ is constant in each iteration. Therefore,
  \[\SC_{\lab}=\tilde{\O}\left(\log(1/\varepsilon)\log(1/\delta)\left(\frac{1}{\varepsilon}\right)^{2\kappa-2}\right). \]	
\end{proof}
\begin{proof}[Proof of Theorem \ref{thm:localadgacagn}.]
The proof follows exactly the same process as that of Theorem \ref{thm:localadgactsy} using $\kappa=1$, and Theorem \ref{thm:adgacagn}.
\end{proof}

%% file: proof_lowerbound.tex
\subsection{Proof of Theorem \ref{thm:labellowerbound}}
\begin{proof}
	Suppose $g(x_1)=a$ and $g(x_0)=b$ for $x_1,x_2\in \mathcal{X}, a<b$. Let $h_1(x)=\sign(g(x)-a)$ and $h_2(x)=\sign(g(x)-b)$. Note that using $Z(x_1,x_2)=0$ incurs $\nu'=0$ for both $h^*=h_1$ and $h^*=h_2$, and thus comparison cannot distinguish between $h_1$ and $h_2$. Suppose $\bC=\{h_1,h_2\}$. Thus, any algorithm $\mathcal{A}$ using both comparison and labeling oracles can be transformed into an algorithm $\mathcal{A}'$ that uses labeling oracle only, by making the comparison oracle always return 0. Note that $\SC_{\lab}(\mathcal{A})=\SC_{\lab}(\mathcal{A}')$, so we only need to lower bound $\SC_{\lab}(\mathcal{A}')$. In the following, we adapt the proof in \cite{hanneke2014theory} to give a lower bound. The main difference is that our goal is to reach a small $\Pr[h(X)\ne h^*(X)]$, whereas in \cite{hanneke2014theory} the goal is a small $\err(h)-\err(h^*)$.
	
	Let $P(x_1)=24\varepsilon, P(x_0)=1-24\varepsilon$. Consider two distributions $\P_1, \P_2$ over $\mathcal{X}\times \mathcal{Y}$ with two different Bayes function $\eta_1(), \eta_2()$. Let $\gamma=\varepsilon^{\kappa-1}$ if $\kappa>1$, or $\gamma=\frac{1}{48}$ if $\kappa=1$. Let $\eta_1(x_0)=\eta_2(x_0)=1$, $\eta_1(x_1)=\frac{1}{2}+\gamma, \eta_2(x_1)=\frac{1}{2}-\gamma$. It is easy to verify both $\P_1$ and $\P_2$ satisfy Tsybakov noise condition.
	
	Choose the groundtruth distribution to be $\P_1$ or $\P_2$ both with probability $1/2$. By the same proof as Theorem 4.3 in \cite{hanneke2014theory}, an event happens with probability at least $\delta$ that $\hat{h}(x_1)\ne h^*(x_1)$, and thus $\Pr[\hat{h}(X)\ne h^*(X)]\geq \varepsilon$, if at most $2\lfloor\frac{1-\gamma^2}{\gamma^2}\log\left(\frac{1}{8\delta(1-2\delta)}\right) \rfloor$ labels are queried.
	So we prove the theorem for TNC.
    
    The proof for adversarial noise is the same as the above proof using $\kappa=1$.
\end{proof}

\subsection{Proof of Theorem \ref{thm:totallowerbound}}
\begin{proof}[Proof of Theorem \ref{thm:totallowerbound}]
The first term in \eqref{equ: lower bound equ} follows directly from Theorem \ref{thm:labellowerbound}. For the second term, we consider the case where both labeling and comparison oracles are perfect with $\nu=\nu'=0$. This is a special case for all Conditions \ref{cond:labeladv}, \ref{cond:labeltsy} and \ref{cond:compadv}. Notice that in this case, a perfect comparison oracle can be constructed from a labeling oracle by $Z(x,x')=\sign(Y(x)-Y(x'))=\sign(h^*(x)-h^*(x'))$;  thus, any algorithm $\mathcal{A}$ with access to both labeling and comparison oracles can be transformed into another algorithm $\mathcal{A}'$ that uses labeling oracle (by replacing the comparison oracle with one that queries labeling oracle instead). So we have
\[2\SC_{\comp}(\mathcal{A})+\SC_{\lab}(\mathcal{A})=\SC_{\lab}(\mathcal{A}')=\Omega(d\log(1/\varepsilon)), \]
where $\Omega(d\log(1/\varepsilon))$ is the standard lower bound for realizable active learning (see e.g., \cite{hanneke2014theory}).
\end{proof}

\subsection{Proof of Theorem \ref{thm:minimaxnup}}
Define $R^B(\hat{g})$ to be the error of comparison oracle induced by $\hat{g}$, and also $\mathbb{C}_{\hat{g}}=\{h:h(x)=\sign(\hat{g}(x)-t),t\in \mathbb{R}\}$. To prove Theorem \ref{thm:minimaxnup}, we first give a lower bound on the left hand side (Theorem \ref{thm:errorlowerbound}) by giving a $\hat{g}$ that every $h\in \bC_{\hat{g}}$ will have every at least $\sqrt{\nu'}$. Then we give an upper bound on it (Theorem \ref{thm:errorbound}) by finding a good estimator $t$. We find $t$ by reducing $\Pr[\sign(\hat{g}(X)-t)\ne h^*(X)]$ to the case when for every $x,x'$ such that $\hat{g}(x)=\hat{g}(x')$ we also have $h^*(x)=h^*(x')$. We find such a good function $f$ in this case by fixing the amount of error at each value of $\hat{g}(x)$, and carefully adjusting the noise levels.
 
\begin{theorem}
	\label{thm:errorlowerbound}
	Suppose $\min\{\Pr[h^*(X)=1], \Pr[h^*(X)=-1]\}\geq \sqrt{\nu'}$. For any $g^*$ such that $g^*(X)$ has a density function, there exists $\hat{g}$ which induces a comparison oracle with error $\nu'$, such that for every $h\in \mathbb{C}_{\hat{g}}$, we have $\Pr[h(X)\ne h^*(X)]\geq \sqrt{\nu'}$.
\end{theorem}
\begin{proof}
	Consider the distribution of $g^*(X)$. Pick a consecutive interval $I=[a,b]
	$ with $a<0<b$ such that $\Pr(g^*(X)\in [0,b])=\Pr(g^*(X)\in [a,0])=\sqrt{\nu'}$. Pick some integer $n\in \mathbb{N}$. Suppose the cdf and pdf of random variable $T=g^*(X)$ is $F(t)$ and $p(t)$ respectively. Define
    \[\hat{g}(x)=\begin{cases}
    a+(b-a)\frac{F(g^*(x))-F(a)}{\sqrt{v'}}, & \text{if } x\in [a,0],\\
    a+(b-a)\frac{F(g^*(x))-F(0)}{\sqrt{v'}}, & \text{if } x\in (0,b],\\    
    g^*(x), & \text{otherwise.}
    \end{cases}\]
    The error of the comparison oracle induced by $\hat{g}$ can be represented as
    \begin{align*}
    R^B(\hat{g})&=2\int_{g^*(x)\in (0,b]} p(g^*(x))\int_{g^*(x')\in [a,0)} p(g^*(x'))\cdot \delta(\hat{g}(x')>\hat{g}(x))\;\text{d} g^*(x) \text{d} g^*(x')\\
    \end{align*}
    Let $t=g^*(x)$ and $t'=g^*(x')$. Then $\hat{g}(x')>\hat{g}(x)$ if and only if 
    \begin{align*}
    & F(t')-F(a)>F(t)-F(0),\\
    \Leftrightarrow &F(t)-F(t')<\sqrt{\nu'}.
    \end{align*}
    For every $t\in [0,b]$, let $G(t)$ satisfy $F(t)-F(G(t))=\sqrt{\nu'}$. Then
     \begin{align*}
    R^B(\hat{g})&=2\int_{t=0}^b p(t)\int_{t'=a}^0 p(t')\cdot \delta\left(F(t)-F(t')<\sqrt{\nu'}\right)\;\text{d} t \text{d} t'\\
    &=2\int_{t=0}^b p(t) \int_{t'=a}^{G(t)} p(t')\;\text{d} t \text{d} t'\\
    &=2\int_{t=0}^b p(t) (F(G(t))-F(a))\text{d} t\\
    &=2\int_{t=0}^b p(t) (F(t)-F(0))\text{d} t\\
    &=2\int_{t=0}^b p(t)\int_{t'=0}^t p(t')\text{d} t\text{d} t'\\
    &=2\int_{t=0}^b\int_{t'=0}^b p(t)p(t')\delta(t'<t)\text{d} t\text{d} t'\\
    &=\nu'.
    \end{align*}
    Now examine any function in $\bC_{\hat{g}}$. If we pick a threshold $t\not\in [a,b]$, the error is at least $\sqrt{\nu'}$ since we incur error on either $\{x:g^*(x)\in [a,0]\}$ or $\{x:g^*(x)\in [0,b]\}$. If we pick threshold $a+(b-a)t$ for $t\in [0,1]$, we induce an error for any $g^*(x)\in [a,0]$ with $\frac{F(g^*(x))-F(a)}{\sqrt{v}}>t$, and any $g^*(x)\in (0,b]$ with $\frac{F(g^*(x))-F(0)}{\sqrt{v}}<t$. A routine calculation shows the error is always $\sqrt{\nu'}$.
	
\end{proof}
\begin{theorem}
	\label{thm:errorbound}
	Suppose that $\hat{g}$ induces a comparison oracle with error $\nu'$, and also distributions of $\hat{g}(X)$ and $g^*(X)$ are smooth in the sense that they both have a density function. 
	There exists $h_t(x)\coloneqq \sign(\hat{g}(x)-t)\in \mathbb{C}_{\hat{g}}$ such that the error of $h_t(x)$ with respect to $h^*(x)$ is at most $\sqrt{\nu'}$, i.e.,
	\[\Pr[h_t(X) \neq h^*(X)] = \Pr[(\hat{g}(X)-t)g^*(X)<0]\leq \sqrt{\nu'}. \]
\end{theorem}
We first prove the inequality:

 \begin{lemma}
 	\label{lemma:inequality}
 	Suppose $\{x_i\}_{i=1}^n$ and $\{y_i\}_{i=1}^n$ satisfies $x_i,y_i\in \mathbb{R}, x_i,y_i\geq 0$. If
 	\(\sum_{i=1}^n \sum_{j=i}^n x_iy_j\leq t, \)
 	we have
 	\[\min_{k=0,1,...,n}\{x_1+\cdots+x_k+y_{k+1}+\cdots+y_n \}\leq \sqrt{\frac{2nt}{n+1}}, \]
 	the equality holds when
 	$x_1=x_2=\cdots=x_n=y_1=\cdots=y_n=\sqrt{\frac{2t}{n(n+1)}}.$
 \end{lemma}

\begin{proof}[Proof of Lemma \ref{lemma:inequality}]
	Let $f(k)=x_1+\cdots+x_k+y_{k+1}+\cdots+y_n$.
	We first prove that when the maximum of $\min_{k=0,1,...,n}f(k)$ is achieved, we must have $x_i=y_i$ for all $i$. If not, not losing generality suppose $x_l>y_l$. Now consider $x_i'=x_i$ for all $i\ne l, l+1$, and $x_l'=y_l, x_{l+1}'=x_{l+1}+x_l-y_l$ (omit the latter step if $l=n$). Let $f'(k)$ be the function of $k$ computed based on $x'$ and $y$. By $x_l>y_l$ we have $f(l)>f(l-1)$. Notice that only $f'(l)=f(l-1)<f(l)$ is reduced and for all other $k\ne l$ we have $f(k)=f'(k)$, so the minimum remains the same. Now we have
	\begin{align*}
	\sum_{i=1}^n \sum_{j=i}^n x_i'y_j&=\sum_{j=1}^n \sum_{i=1}^j x_i'y_j =\sum_{j=1}^n y_j\sum_{i=1}^j x_i'\\
	&=\sum_{j=1}^{l-1} y_j\sum_{i=1}^j x_i+y_l\sum_{i=1}^{l} x_i'+\sum_{j=l+1}^n y_j \sum_{i=1}^j x_i\\
	&\leq \sum_{j=1}^{l-1} y_j\sum_{i=1}^j x_i+y_l\sum_{i=1}^{l} x_i+\sum_{j=l+1}^n y_j \sum_{i=1}^j x_i\\
	&\leq t.
	\end{align*}
	So there exists a configuration that maximizes $\min_k f(k)$ with $x_i=y_i$ for all $i$. Now suppose $x_i=y_i$ for all $i$. The constraint becomes
	\[\sum_{i=1}^n \sum_{j=i}^{n} x_ix_j\leq \varepsilon, \]
	which is equivalent to
	\[\left(\sum_{i=1}^n x_i\right)^2+\sum_{i=1}^n x_i^2\leq 2\varepsilon. \]
	By Cauchy-Schwarz inequality we have
	\[\sum_{i=1}^n x_i^2\geq \frac{\left(\sum_{i=1}^n x_i\right)^2}{n}. \]
	So
	\[x_1+\cdots+x_k+y_{k+1}+\cdots+y_n=\sum_{i=1}^n x_i\leq \sqrt{\frac{2nt}{n+1}}. \]
	It is easy to verify the equality condition.
\end{proof}

\begin{proof}[Proof of Theorem \ref{thm:errorbound}]
	Not losing generality, suppose $\hat{g}(x)\in [0,1]$; such a assumption is justifiable since any increasing transformation of $\hat{g}$ does not change $R^B(\hat{g})$. So we only need to consider $\mathbb{C}_{\hat{g}}=\{h:h(x)=h_t(x)=\sign(\hat{g}(x)-t),t\in [0,1] \}$. Let $q(u)$ denote the distribution of $\hat{g}(X)$. Let $\xi(u)=q(u)\Pr(h^*(X)=1|\hat{g}(x)=u)$. So we have
	\[\int_{0}^t \xi(u)du=\Pr(h^*(X)=1,\hat{g}(X)<t).  \]
	So the error of $h_t$ with respect to $h^*$ can be expressed as
	\begin{align*}
	\Pr((\hat{g}(X)-t)g^*(X)<0)&=\Pr(\hat{g}(X)>t,g^*(X)<0)+\Pr(\hat{g}(X)<t,g^*(X)>0)\\
	&=\int_{0}^t \xi(u)du+\int_{t}^1 (q(u)-\xi(u))du.
	\end{align*}
	On the other hand, the comparison error can be expressed as
	\begin{align*}
	R^B(\hat{g})&=2\Pr(\hat{g}(X)>\hat{g}(X'),h^*(X)=-1,h^*(X)=1)\\
	&=\int_{0}^1 \xi(u)\int_{u}^1 (q(v)-\xi(v))dudv.
	\end{align*}
	Now consider we do this on the grid with step size $1/n$ and let $n\rightarrow \infty$; the integral will be the limit value. So, let
	\[S_n=\frac{1}{n^2}\sum_{i=1}^n \sum_{j=i}^n \xi(i/n)(q(j/n)-\xi(j/n)). \]
	So
	\[\Pr(\hat{g}(X)>g(X'),h^*(X)=-1,h^*(X)=1)=\lim_{n\rightarrow \infty} S_n. \]
	Also, let
	\[T^t_n=\frac{1}{n}\left(\sum_{i: i/n<t} \xi(i/n)+ \sum_{i: i/n>=t} (q(i/n)-\xi(i/n))\right), \]
	so
	\[\Pr((\hat{g}(X)-t)g(X)<0)= \lim_{n\rightarrow \infty} T^t_n.\]
	Now let $x_i=\frac{1}{n} \xi(i/n), y_i=\frac{1}{n}(q(i/n)-\xi(i/n))$ in Lemma
	\ref{lemma:inequality}, and we have
	\[\min_t T^t_n\leq \sqrt{\frac{2nS_n}{n+1}}. \]
	Note that $\lim_{n\rightarrow \infty} 2S_n = R^B(\hat{g})\leq \nu'$ and let $n\rightarrow \infty$ on both side, we have
	\[\min_t  \Pr[(\hat{g}(X)-t)g(X)<0]\leq \sqrt{\nu'}. \]
	
\end{proof}